\documentclass[twoside,11pt]{article}

\usepackage{amsmath,amsfonts,amsthm,amssymb,bibentry,bbm,xspace,ulem,enumitem,euscript,algorithm,algorithmic,graphicx,xcolor,subfigure,url,wrapfig,mathrsfs}

\RequirePackage{epsfig}
\RequirePackage{amssymb}
\RequirePackage{natbib}
\RequirePackage{graphicx}
\bibpunct{(}{)}{;}{a}{,}{,}

\renewcommand{\emph}{\textit}
\definecolor{darkblue}{rgb}{0.0,0.0,0.5}
\newcommand{\comment}[1]{}

\newcommand{\w}{\boldsymbol{w}}
\newcommand{\W}{\boldsymbol{W}}

\renewcommand{\v}{\boldsymbol{v}}

\renewcommand{\a}{\boldsymbol{a}}
\newcommand{\e}{\boldsymbol{e}}

\newcommand{\Hilb}{{\mathcal{H}}}
\renewcommand{\L}{{\mathfrak L}}
\newcommand{\Real}{{\mathbb R}}
\newcommand{\Reg}{{\mathfrak R}}

\newcommand{\R}{{\mathbb R}}
\renewcommand{\v}{\boldsymbol{v}}

\renewcommand{\a}{{\boldsymbol{a}}}

\newcommand{\X}{\mathcal{X}}
\newcommand{\Y}{\mathcal{Y}}
\newcommand{\loss}{l}
\newcommand{\vtheta}{{\boldsymbol{\theta}}}

\newcommand{\valpha}{{\boldsymbol{\alpha}}}

\newcommand{\C}{\bf{C}}

\newcommand{\dom}{\text{\rm dom}}

\newcommand{\zero}{\mathbf{0}}

\newcommand{\norm}[1]{\left\Vert#1\right\Vert}
\newcommand{\bracket}[1]{\left(#1\right)}
\newcommand{\set}[1]{\left\{#1\right\}}
\newcommand{\inner}[1]{\left\langle#1\right\rangle}

\newtheorem{theorem}{Theorem}
\newtheorem{definition}[theorem]{Definition}
\newtheorem{remark}[theorem]{Remark}
\newtheorem{corollary}[theorem]{Corollary}
\newtheorem{proposition}[theorem]{Proposition}
\newtheorem{lemma}[theorem]{Lemma}
\newtheorem{problem}[theorem]{Problem}

\DeclareMathOperator*{\argmin}{argmin\xspace}
\DeclareMathOperator*{\arginf}{arginf\xspace}
\DeclareMathOperator*{\argsup}{argsup\xspace}
\DeclareMathOperator*{\tr}{tr\xspace}

\newcommand{\hide}[1]{}

\begin{document}

\title{Framework for Multi-task Multiple Kernel Learning and Applications in Genome Analysis}

\author{Christian Widmer\thanks{Parts of this work was done while CW was at the Friedrich Miescher Laboratory, Spemannstr. 39, 72076 T\"ubingen, Germany and also with the Machine Learning Group, Technische Universit\"at Berlin, Franklinstr. 28/29, 10587 Berlin, Germany.}\\
       Computational Biology Center \\
       Memorial Sloan Kettering Cancer Center \\
       1275 York Avenue, Box 357, New York, NY 10065, USA \bigskip\\
       Marius Kloft\thanks{Most parts of the work was done while MK was with the Computational Biology Center of Memorial Sloan Kettering Cancer Center (1275 York Avenue, New York, NY 10065, USA) and the Courant Institute of Mathematical Sciences (251 Mercer Street, New York, NY 10012, USA).}\\
       Department of Computer Science\\
       Humboldt University of Berlin \\
       Unter den Linden 6\\
       10099 Berlin, Germany
       \bigskip\\
       Vipin T. Sreedharan \\
       Computational Biology Center \\
       Memorial Sloan Kettering Cancer Center \\
       1275 York Avenue, New York, NY 10065, USA\bigskip\\
       Gunnar R\"atsch\thanks{To whom correspondence should be addressed. Email: \texttt{Gunnar.Ratsch@ratschlab.org}} \\
       Computational Biology Center \\
       Memorial Sloan Kettering Cancer Center \\
       1275 York Avenue, New York, NY 10065, USA
}

\maketitle

\begin{abstract}

%
We present a general regularization-based framework for Multi-task learning (MTL), 
in which the similarity between tasks can be learned or refined using $\ell_p$-norm 
Multiple Kernel learning (MKL). 
Based on this very general formulation (including a general loss function), we derive the
corresponding dual formulation using Fenchel duality applied to Hermitian matrices.
We show that numerous established MTL methods can be derived as special cases from
both, the primal and dual of our formulation.
Furthermore, we derive a modern dual-coordinate descend optimization strategy for
the hinge-loss variant of our formulation and provide convergence bounds for our
algorithm. As a special case, we implement in C++ a fast LibLinear-style solver for $\ell_p$-norm MKL.
%
%
In the experimental section, we analyze various aspects of our algorithm such 
as predictive performance and ability to reconstruct 
task relationships on biologically inspired synthetic data, where we have full control over the underlying ground truth.
We also experiment on a new dataset from the domain of computational biology that we collected for the purpose of this paper. 
It concerns the prediction of transcription start sites (TSS)  over nine organisms, which is a crucial task in gene finding.
Our solvers including all discussed special cases are made available as open-source software as part of the SHOGUN machine learning toolbox (available at \url{http://shogun.ml}).

\end{abstract}

\section{Introduction}\label{introduction}

%

One of the key challenges in computational biology is to build effective and efficient
statistical models that learn from data to predict, analyze, and ultimately
understand biological systems.
Regardless of the problem at hand, however, be it the recognition of sequence
signals such as splice sites, the prediction of protein-protein
interactions, or the modeling of metabolic networks, we frequently have
access to data sets for \emph{multiple} organisms, tissues or cell-lines.
Can we develop methods that optimally combine such multi-domain data?

While the field of Transfer or Multitask Learning enjoys a growing interest
in the Machine Learning community in recent years, it can be traced back to ideas
from the mid 90's. 
%
During that time \cite{thrun1996learning} asked the provocative question 
"Is Learning the $n$-th Thing any Easier Than Learning the First?",
effectively laying the ground for the field of Transfer Learning.
Their work was motivated by findings in human psychology, where humans
were found to be capable of learning based on as few as a single
example \citep{ahn1993psychological}.
The key insight was that humans build upon previously learned related concepts, when learning new tasks, something \cite{thrun1996learning} call \emph{lifelong learning}.
Around the same time, \cite{DBLP:conf/icml/Caruana93,Caruana1997} coined the term \emph{Multitask Learning}.
Rather than formalizing the idea of learning a sequence of tasks, 
they propose machinery to learn multiple related tasks in parallel.

While most of the early work on Multitask Learning was carried out in
the context of learning a shared representation for neural networks
\citep{Caruana1997,Baxter2000}, \cite{Evgeniou2004} adapted this
concept in the context of kernel machines.
At first, they assumed that the models of all tasks are close to each
other \citep{Evgeniou2004} and later generalized their framework to
non-uniform relations, allowing to couple some tasks more strongly
than others \citep{evgeniou2006learning}, according to some externally
defined task structure.
%
In recent years, there has been an increased interest in learning the structure potentially underlying the tasks.
\cite{Ando2005} proposed a non-convex method based on 
Alternating Structure Optimization (ASO) for identifying 
the task structure. A convex relaxation of their approach
was developed by \cite{Chen2009}. \cite{Zhou2011}
showed the equivalence between ASO and Clustered Multitask 
Learning \citep{jacob2008clustered,obozinski2010joint} and
their convex relaxations.
While the structure between tasks is defined by assigning tasks to
clusters in the above approaches, \cite{Zhang2010} propose to learn a
constrained task covariance matrix directly and show the relationship
to Multitask Feature Learning
\citep{Argyriou2007,argyriou2008convex,argyriou2008spectral,Liu2009}.
Here, the basic idea is to use a LASSO-inspired \citep{Tibshirani1996}
$\ell_{2,1}$-norm to identify a subset of features that is relevant to
all tasks.

A challenge remains to find an adequate task similarity measure to
compare the multiple domains and tasks.  While existing parameter-free
approaches such as \cite{romera2013multilinear} ignore biological
background knowledge about the relatedness of the tasks, in this
paper, we present a parametric framework for regularization-based
multitask learning that subsumes several approaches and automatically
learns the task similarity from a set of candidates measures using
$\ell_p$-norm Multiple Kernel learning (MKL) see, for instance,
\citet{KloBreSonZie11}.
We thus provide a middle ground between assuming known task
relationships and learning the entire task structure from scratch.
We propose a general unifying framework of MT-MKL, including a
thorough dualization analysis using Fenchel duality, based on which we
derive an efficient linear solver that combines our general framework
with advances in linear SVM solvers and evaluate our approach on
several datasets from Computational Biology.

This paper is based on preliminary material shown in several
conference papers and workshop contributions
\citep{widmer2010leveraging,Widmer2010,widmer2010novel,Widmer2012,widmer2012multitask},
which contained preliminary aspects of the framework presented here.
This version additionally includes a unifying framework including
Fenchel duality analysis, more complete derivations and theoretical
analysis as well as a comparative study in multitask learning and
genomics, where we brought together genomic data for a wide range of
biological organisms in a multitask learning setting.  This dataset
will be made freely available and may serve as a benchmark in the
domain of multitask learning.  Our experiments show that combining
data via multitask learning can outperform learning each task
independently. In particular, we find that it can be crucial to
further refine a given task similarity measure using multitask
multiple kernel learning.

The paper is structured as follows: In Section~\ref{graph_mtl} we
introduce a unifying view of multitask multiple kernel learning that
covers a wide range loss functions and regularizers. We give a general
Fenchel dual representation and a representer theorem, and show that
the formulation contains several existing formulations as special
cases. In Section~\ref{sec:op} we propose two optimization strategies:
one that can be applied out of the box with any custom set of kernels
and another one that is specifically tailored to linear kernels as
well as string kernels. Both algorithms were implemented into the
Shogun machine learning toolbox. In Section~\ref{experiments} we
present results of empirical experiments on artificial data as well as
a large biological multi-organism dataset curated for the purpose of
this paper.

\section{A Unifying View of Regularized Multi-Task Learning}\label{graph_mtl}

In this section, we present a novel multi-task framework comprising
many existing formulations, allowing us to view prevalent approaches
from a unifying perspective, yielding new insights.  We can also
derive new learning machines as special instantiations of the general
model. Our approach is embedded into the general framework of
regularization-based supervised learning methods, where we minimize a
functional
$$ \Reg(\w) + C\thinspace\L(\w) \thinspace,$$ which consists of a
loss-term $\L(\w)$ measuring the training error and a regularizer
$\Reg(\w)$ penalizing the complexity of the model $\w$.  The positive
constant $C>0$ controls the trade-off of the criterion.
The formulation can easily be generalized to the multi-task setting,
where we are interested in obtaining several models parametrized by
$\w_1, \ldots, \w_T$, where $T$ is the number of tasks.

In the past, this has been achieved by employing a joint
regularization term $\Reg(\w_1, \ldots, \w_T)$ that penalizes the
discrepancy between the individual models
\citep{evgeniou2006learning,agarwal2010learning},
$$ \Reg(\w_1, \ldots, \w_T) +
C\thinspace\L(\w_1,\ldots,\w_t)\thinspace.$$ A common approach is, for
example, to set $ \Reg(\w_1, \ldots, \w_T) =
\frac{1}{2}\sum\nolimits_{s,t=1}^Tq_{st}\norm{\w_s-\w_t}^2\thinspace,$
where $Q=\bracket{q_{st}}_{a\leq s,t\leq T}$ is a task similarity
matrix.  In this paper, we develop a novel, general framework for
multi-task learning of the form
$$\min_{\W,\theta} ~ \Reg(\W,\vtheta) + C\L(\W) \thinspace,$$ 
where $\W=(W_m)_{1\leq m\leq M}$, $W_m=(\w_{m1}, \ldots,
\w_{mT})$. This approach has the additional flexibility of allowing us
to incorporate \emph{multiple task similarity matrices} into the
learning problem, each equipped with a weighting factor. Instead of
specifying the weighting factor \emph{a priori}, we will automatically
determine optimal weights from the data as part of the learning
problem. We show that the above formulation comprises many existing
lines of research in the area; this not only includes very recent
lines but also seemingly different ones.  The unifying framework
allows us to analyze a large variety of MTL methods jointly, as
exemplified by deriving a general dual representation of the
criterion, without making assumptions on the employed norms and
losses, besides the latter being convex. This delivers insights into
connections between existing MTL formulations and, even more
importantly, can be used to derive \emph{novel} MTL formulations as
special cases of our framework, as done in a later section of this
paper.

\subsection{Problem Setting and Notation}

Let $D=\{(x_1,y_1),\ldots,(x_n,y_n)\}$ be a set of training
pattern/label pairs.  In multitask learning, each training example
$(x_i,y_i)$ is associated with a task $\tau(i)\in\{1,\ldots,T\}$.
Furthermore, we assume that for each $t\in\{1,\ldots,T\}$ the
instances associated with task $t$ are independently drawn from a
probability distribution $P_t$ over a measurable space
$\X_t\times\Y_t$.  We denote the set of indices of training points of
the $t$th task by $I_t:=\{i\in\{1,\ldots,n\}:\tau(i)=t\}$.  The goal
is to find, for each task $t\in\{1,\ldots,T\}$, a prediction function
$f_t:\X\rightarrow\R$.  In this paper, we consider composite functions
of the form
$f_t:x\mapsto\sum_{m=1}^M\langle\w_{mt},\varphi_m(x)\rangle$, $1\leq
t\leq T$, where $\varphi_m:\X\rightarrow\Hilb_m$, $1\leq m\leq M$, are
mappings into reproducing Hilbert spaces $\Hilb_1,\ldots,\Hilb_M$,
encoding multiple views of the multi-task learning problem via kernels
$k_m(x,\tilde{x})=\langle\varphi_m(x),\varphi_m(\tilde{x})\rangle$,
and $\W:=(\w_{mt})_{1\leq m\leq M,\thinspace1\leq t\leq T}$,
$w_{mt}\in\Hilb_m$ are parameter vectors of the prediction function.

For simplicity of notation, we concentrate on binary prediction, i.e.,
$\Y=\{-1,1\}$, and encode the loss of the prediction problem as a loss
term $\L(\W):=\sum_{i=1}^nl(y_if_{\tau(i)}(x_i))$, where $\loss:
\mathbb{R} \rightarrow \R\cup\{\infty\}$ is a loss function, assumed
to be closed convex, lower bounded and finite at $0$. To consider
sophisticated couplings between the tasks, we introduce so-called
\emph{task-similarity matrices} $Q_1,\ldots,Q_M\in\textrm{GL}_n(\R)$
with $Q_m=(q_{mst})_{1\leq s,t\leq T}$,
\thinspace$Q^{-1}_m=\big(q_{mst}^{(-1)}\big)_{1\leq s,t\leq T}$ and
consider the regularizer
$\Reg_{\vtheta}(\W)=\frac{1}{2}\sum_{m=1}^M\Vert
W_m\Vert^2_{Q_m}/\theta_m$ (setting $1/0:=\infty$, $0/0:=0$) with
$\norm{W_m}_{Q_m}:=\tr(W_mQ_mW_m^*)=\sqrt{\sum_{s,t=1}^Tq_{mst}\inner{\w_{ms},\w_{mt}}}$,
where
$W_m=(\w_{m1},\ldots,\w_{mT})\in\bigoplus_{t=1}^T\Hilb_m=:\Hilb_m^T$
with adjoint $W_m^*$ and $\tr(\cdot)$ denotes the trace class operator
of the tensor Hilbert space $\Hilb_m\otimes\Hilb_m$.  Note that also
the direct sum $\Hilb:=\bigoplus_{m=1}^M\Hilb_m^T$ is a Hilbert space,
which will allow us to view $\W\in\Hilb$ as an element in a Hilbert
space.  The parameters $\vtheta=(\theta_m)_{1\leq m\leq
  M}\in\Theta_p$, $\Theta_p:=\{\vtheta\in\R^M: \theta_m\geq0, 1\leq
m\leq M, \norm{\vtheta}_p\leq1\}$, are adaptive weights of the views,
where $\norm{\vtheta}_p=\sqrt[p]{\sum_{m=1}^M |\theta_m|^p}$ denotes
the $\ell_p$-norm.  Here $\vtheta\succeq\zero$ denotes
$\theta_m\geq0$, $m=1,\ldots,M$.

Using the above specification of the regularizer and the loss term, we
study the following unifying primal optimization problem.

\begin{problem}[\rm Primal problem]\label{prob:primal}
  Solve
  $$ \inf_{\vtheta\in\Theta_p,\W\in\Hilb}  \quad \Reg_{\vtheta}(\W)  ~+~ C\thinspace\L\big(A(\W)\big) \thinspace, $$
  where
  \begin{align*} 
    &\Reg_{\vtheta}(\W) := \frac{1}{2}\sum_{m=1}^M\frac{\Vert W_m\Vert^2_{Q_m}}{\theta_m} \thinspace,  \quad  \Vert W_m\Vert^2_{Q_m}:=\tr(W_mQ_mW_m^*) \\
    &\L(A(\W)) := \sum_{i=1}^nl\big(A_i(\W)\big)\thinspace,  ~~  A(\W):=\bracket{A_i(\W)}_{1\leq i\leq n}\thinspace, ~~  A_i(\W):=y_i\sum_{m=1}^M\big\langle\w_{m\tau(i)},\varphi_m(x_i)\big\rangle\thinspace. 
  \end{align*}
\end{problem}

\subsection{Dualization}\label{sec:dualization}

Dual representations of optimization problems deliver insight into the
problem, which can be used in practice to, for example, develop
optimization algorithms (so done in Section~\ref{sec:op} of this
paper). In this section, we derive a dual representation of our
unifying primal optimization problem, i.e.,
Problem~\ref{prob:primal}. Our dualization approach is based on
Fenchel-Rockafellar duality theory.  The basic results of
Fenchel-Rockafellar duality theory for Hilbert spaces are reviewed in
Appendix~\ref{app:fenchel}.  We present two dual optimization
problems: one that is dualized with respect to $\W$ only (i.e.,
considering $\vtheta$ as being fixed) and one that completely removes
the dependency on $\vtheta$.

\subsubsection{Computation of Conjugates and Adjoint Map}

To apply Fenchel's duality theorem, we need to compute the adjoint map
$A^*$ of the linear map $A:\Hilb\rightarrow\R^n$,
$A(\W)=\big(A_i(\W)\big)_{1\leq i\leq n}$, as well as the convex
conjugates of $\Reg$ and $\L$.  See Appendix~\ref{app:fenchel} for a
review of the definitions of the convex conjugate and the adjoint map.
First, we notice that, by the basic identities for convex conjugates
of Prop.~\ref{eq:comp_rules} in Appendix \ref{app:fenchel}, we have
that
$$ \big(C\L(\valpha)\big)^* ~ = ~ C\L^*(\valpha/C) ~ = ~ C
\Big(\sum\nolimits_{i=1}^n \loss(\alpha_i/C)\Big)^* ~ = ~ C
\sum\nolimits_{i=1}^n \loss^*(\alpha_i/C) \thinspace.$$ 
Next, we define $A^*:\R^n\rightarrow\Hilb$ by
$A^*(\valpha)=\big(\sum_{i\in
  I_t}\alpha_iy_i\varphi_m(x_i)\big)_{1\leq m\leq M,1\leq t\leq
  T}$\thinspace.  Recall that the mapping between tasks and examples
may be expressed in one of two ways.  We may use index set $I_t$ to
retrieve the indices of training examples associated with task $t$.
Alternatively, we may use task indicator $\tau(i) \in \{1,\dots,T\}$
to obtain the task index $\tau(i)$ associated with $i$th training
example.  Using this notation, we verify that, for any $\W\in\Hilb$
and $\valpha\in\R^n$, it holds
\begin{eqnarray*}
  \inner{\W,A^*(\valpha)} &=& \Big\langle\big(w_{mt}\big)_{1\leq m\leq M,1\leq t\leq T}~,\Big(\sum\nolimits_{i\in I_t}\alpha_iy_i\varphi_m(x_i)\Big)_{1\leq m\leq M,1\leq t\leq T}\Big\rangle \\
         &=& \sum_{m=1}^M\sum_{t=1}^T \sum_{i\in I_t} \alpha_iy_i \inner{w_{mt},\varphi_m(x_i)} \\
         &=& \sum_{i=1}^n\sum_{m=1}^M \alpha_iy_i \inner{w_{m\tau(i)},\varphi_m(x_i)} \\
       &=& \inner{A(\W),\valpha}\thinspace.
\end{eqnarray*}
Thus, $A^*$ as defined above is indeed the adjoint map.  Finally, we
compute the conjugate of $\Reg$ with respect to $\W$, where we
consider $\vtheta$ as a constant (be reminded that $Q_m$ are given).
We write $r_m(W_m):=\frac{1}{2}\norm{W_m}_{Q_m}^2$ and note that, by
Prop.~\ref{eq:comp_rules},
\begin{eqnarray*}
  \Reg_{\vtheta}^*(\W) &=& \bracket{\sum_{m=1}^M\theta_m^{-1}r_m(W_m)}^*  ~=~  \sum_{m=1}^M\theta_m^{-1}r_m^*(\theta_mW_m) \thinspace.
\end{eqnarray*}
Furthermore, 
\begin{equation}\label{eq:gm_aux}
  r^*_m(W_m) = \sup_{V_m\in\Hilb_m^T} \underbrace{\langle V_m,W_m\rangle - \frac{1}{2}\tr(V_mQ_mV_m)}_{=:\psi(V_m)} \thinspace.
\end{equation}  
The supremum is attained when $\nabla_{V_m}\psi(V_m)=0$ so that in the optimum $V_m=Q_m^{-1}W_m$. Resubstitution into \eqref{eq:gm_aux} gives
$r_m^*(W_m) = \frac{1}{2} \tr(W_mQ_m^{-1}W_m)=\frac{1}{2}\norm{W_m}_{Q^{-1}_m}^2$, so that we have
$$\Reg_{\vtheta}^*(\W) = \frac{1}{2}\sum_{m=1}^M \theta_m\norm{W_m}_{Q_m^{-1}}^2 \thinspace.$$


\subsubsection{Dual Optimization Problems}

\noindent We may now apply Fenchel's duality theorem
(cf.\ Theorem~\ref{thm:fenchel} in Appendix~\ref{app:fenchel}), which
gives the following dual MTL problem:

\begin{problem}[\rm Dual problem---partially dualized minimax formulation]\label{prob:dual}
Solve
\begin{equation}\label{eq:partial_dual_prob}
  \inf_{\vtheta\in\Theta_p} ~ \sup_{\valpha\in\R^n} ~\thinspace -\Reg_{\vtheta}^*(A^*(\valpha)) ~-~ C\thinspace \L^*(-\valpha/C)  \thinspace, 
\end{equation}
where 
\begin{align}\label{eq:part_dual_constr}
\begin{split}
   & \Reg_{\vtheta}^*(A^*(\valpha)) = \frac{1}{2}\sum_{m=1}^M\theta_m\norm{A^*_m(\valpha)}_{Q_m^{-1}}^2 \thinspace, \quad 
        \L^*(\valpha)  =  \sum_{i=1}^nl^*(\alpha_i)\thinspace,\\
   & A^*(\valpha):=\bracket{A^*_m(\valpha)}_{1\leq m\leq M}\thinspace, \quad A^*_m(\valpha)=\Big(\sum\nolimits_{i\in I_t}\alpha_iy_i\varphi_m(x_i)\Big)_{1\leq t\leq T}  \thinspace. 
\end{split}
\end{align}
\end{problem}
The above problem involves minimization with respect to (the primal variable) $\vtheta$ and maximization with respect to (the dual variable) $\valpha$. 
The optimization algorithm presented later in this paper will optimize is based on this minimax formulation. 
However, we may completely remove the dependency on $\vtheta$, which sheds further insights into the problem, which 
will later be exploited for optimization, i.e., to control the duality gap of the computed solutions.

To remove the dependency on $\vtheta$, we first note that Problem~\ref{prob:dual} is convex (even affine) in $\vtheta$ and concave in $\valpha$ and 
thus, by Sion's minimax theorem, we may exchange the order of minimization and maximization:
\begin{eqnarray*}
  \textrm{Eq.}\medspace\eqref{eq:partial_dual_prob} &=& \inf_{\vtheta\in\Theta_p} ~ \sup_{\valpha\in\R^n} ~\thinspace -\frac{1}{2}\sum_{m=1}^M\theta_m\norm{A^*_m(\valpha)}_{Q_m^{-1}}^2 ~-~ C\thinspace \L^*(-\valpha/C) \\
    &=& \sup_{\valpha\in\R^n} ~ -\sup_{\vtheta\in\Theta_p} ~\thinspace \frac{1}{2}\sum_{m=1}^M\theta_m\norm{A^*_m(\valpha)}_{Q_m^{-1}}^2 ~+~ C\thinspace \L^*(-\valpha/C) \\ 
    &=& \sup_{\valpha\in\R^n} ~ -\frac{1}{2}\bigg\Vert\bracket{\norm{A^*_m(\valpha)}_{Q_m^{-1}}^2}_{1\leq m\leq M}\bigg\Vert_{p^*} ~+~ C\thinspace \L^*(-\valpha/C) 
\end{eqnarray*}
where the last step is by the definition of the dual norm, i.e., 
$\sup_{\vtheta\in\Theta_p} \big\langle\vtheta,\widetilde{\vtheta}\big\rangle = \big\Vert\widetilde{\vtheta}\big\Vert_{p^*}$ and $p^*:=p/(p-1)$ denotes the conjugated exponent. We thus have the following alternative dual problem.
\begin{problem}[\rm Dual problem---completely dualized formulation]\label{prob:complete_dual}
Solve
$$   \sup_{\valpha\in\R^n} ~ -\frac{1}{2}\bigg\Vert\bracket{\norm{A^*_m(\valpha)}_{Q_m^{-1}}^2}_{1\leq m\leq M}\bigg\Vert_{p^*} ~+~ C\thinspace \L^*(-\valpha/C) $$
where 
$$ \L^*(\valpha)  =  \sum_{i=1}^nl^*(\alpha_i)\thinspace, \quad  A^*_m(\valpha)=\Big(\sum\nolimits_{i\in I_t}\alpha_iy_i\varphi_m(x_i)\Big)_{1\leq t\leq T}  \thinspace. $$
\end{problem}

\subsection{Representer Theorem}\label{sec:representer}

Fenchel's duality theorem (Theorem~\ref{thm:fenchel} in Appendix~\ref{app:fenchel}) yields 
a useful optimality condition, that is,
$$  (\W^\star,\valpha^\star)\text{ \thinspace optimal} ~~ \Leftrightarrow ~~ \W^\star=\nabla g^*(A^*(\valpha^\star)) \thinspace, $$
under the minimal assumption that $g\circ A^*$ is differentiable in $\valpha^\star$.
The above requirement can be thought of as an analog to the KKT condition \emph{stationarity} in Lagrangian duality. 
Note that we can rewrite the above equation by inserting the definitions of $g$ and $A$ from the previous subsection; this gives, for any $m=1,\ldots,M$,
$$ \forall m=1,\ldots,M: \quad W_m^\star = \theta_mQ_m^{-1}\bracket{\sum\nolimits_{i\in I_t}\alpha_i^\star y_i\varphi_m(x_i)}_{1\leq t\leq T} \thinspace,$$
which we may rewrite as
\begin{equation}\label{eq:representer}
  \forall \thinspace m=1,\ldots,M,\thinspace t=1,\ldots,T: \quad \w_{mt}^\star = \theta_m\sum_{i=1}^nq^{(-1)}_{m\tau(i)t}\alpha^\star_iy_i\varphi_m(x_i) \thinspace.
\end{equation}
The above equation gives us a representer theorem \citep{argyriou2009there} for the optimal $\W^\star$, which we will exploit later in this paper for deriving an efficient optimization algorithm
to solve Problem~\ref{prob:primal}.

\subsection{Relation to Multiple Kernel Learning}\label{sec:rel-mkl}

\cite{evgeniou2006learning} introduce the notion of a \emph{multi-task kernel}. We can generalize this framework by
defining multiple multi-task kernels
\begin{equation}\label{eq:def-mtk}
  \tilde{k}_m(x_i,x_j) := q^{(-1)}_{m\tau(i)\tau(j)}k_m(x_i,x_j) \thinspace, ~~ m=1,\ldots,M \thinspace.
\end{equation}
To see this, first note that the term $\norm{A^*_m(\valpha)}^2_{Q^{-1}}$ can alternatively be written as
\begin{eqnarray}\label{eq:mtk-norm-term}
  \begin{split}
  \norm{A^*_m(\valpha)}^2_{Q_m^{-1}} &= \tr\bracket{A^*_m(\valpha)\thinspace Q_m^{-1}\thinspace A^*_m(\valpha)^*} \\
    &= \tr\bracket{\bracket{\sum\nolimits_{i\in I_s}\alpha_iy_i\varphi_m(x_i)}_{1\leq s\leq T} Q_m^{-1}\bracket{\sum\nolimits_{i\in I_t}\alpha_iy_i\varphi_m(x_i)}_{1\leq t\leq T}^*} \\
    &= \sum_{s,t=1}^Tq_{mst}^{(-1)}\inner{\sum\nolimits_{i\in I_s}\alpha_iy_i\varphi_m(x_i),\sum\nolimits_{i\in I_t}\alpha_iy_i\varphi_m(x_i)} \\   
    &= \sum_{s,t=1}^T q_{mst}^{(-1)}\sum\nolimits_{i\in I_s,j\in I_t}\alpha_i\alpha_jy_iy_j\underbrace{\varphi_m(x_i)\varphi_m(x_j)}_{=\thinspace k_m(x_i,x_j)}  \\
    &=  \sum_{i,j=1}^n \alpha_i\alpha_jy_iy_j\underbrace{q_{m\tau(i)\tau(j)}^{(-1)}k_m(x_i,x_j)}_{\tilde{k}_m(x_i,x_j)} \thinspace. 
  \end{split}
\end{eqnarray}
so it follows
\begin{equation*}
  \Reg_{\vtheta}^*(A^*(\valpha)) ~=~ \frac{1}{2}\sum_{i,j=1}^n \alpha_i\alpha_jy_iy_j\sum_{m=1}^M\theta_m\tilde{k}_m(x_i,x_j) 
\end{equation*}
and thus Problem~\ref{prob:dual} becomes
\begin{equation}\label{eq:prob-mtk}
  \inf_{\vtheta\in\Theta_p} ~ \sup_{\valpha\in\R^n} ~\thinspace -\frac{1}{2}\sum_{i,j=1}^n \alpha_i\alpha_jy_iy_j\sum_{m=1}^M\theta_m\tilde{k}_m(x_i,x_j) 
      ~-~ C\thinspace \L^*(-\valpha/C)  \thinspace,  
\end{equation}
which is an $\ell_p$-regularized multiple-kernel-learning problem over the kernels $\tilde{k}_1,\ldots,\tilde{k}_M$ \citep{KloBreLasSon08,KloBreSonZie11}.

\begin{table}[t]
\begin{center}
\small
\medskip
\begin{tabular}{|l|l|l|}
  \hline
  & loss $l(a), a\in\R$  &  dual loss $l^*(a)$  \\\hline
  hinge loss   &    $\max(0,1-a)$     &   $\begin{cases} a\text{, \thinspace if }-1\leq a\leq 0 \\ \infty\text{, \thinspace elsewise} \end{cases}$  \\[4pt]
    logistic loss &  $\log(1+\exp(-a)$  &   $\begin{cases} -a\log(-a)+(1+a)\log(1+a)\text{, \thinspace if }-1\leq a\leq 0 \\ \infty\text{, \thinspace elsewise} \end{cases}$  \\\hline
\end{tabular}
\bigskip
\caption{\label{tab:Fenchel} Examples of loss functions and corresponding conjugate functions. See Appendix~\ref{app:loss}.}
\end{center}
\end{table}

\subsection{Specific Instantiations of the Framework}

In this section, we show that several regularization-based multi-task learning machines are subsumed by the generalized 
primal and dual formulations of Problems~\ref{prob:primal}--\ref{prob:dual}. 
As a first step, we will specialize our general framework to the hinge-loss,
and show its primal and dual form.
Based on this, we then instantiate our framework further to known methods in increasing complexity,
starting with single-task learning (standard SVM) 
and working towards graph-regularized multitask learning and its relation
to multitask kernels.
Finally, we derive several novel methods from our general framework.
%
%
%
%


%
\subsubsection{Hinge Loss}\label{sec:hinge}

Many existing multi-task learning machines utilize the hinge loss $l(a)=\max(0,1-a)$. Employing the hinge loss in Problem~\ref{prob:primal}, yields the loss term
$$ \L(A(\W)) = \sum_{i=1}^n \max\bracket{0,1-y_i\sum\nolimits_{m=1}^M\big\langle\w_{m\tau(i)},\varphi_m(x_i)\big\rangle} \thinspace. $$
Furthermore, as shown in Table~\ref{tab:Fenchel}, the conjugate of the hinge loss is $l^*(a)=a$, if $-1\leq a\leq 0$ and $\infty$ elsewise, which
is readily verified by elementary calculus. Thus, we have
\begin{equation}\label{eq:dualloss_hinge}
  -C\thinspace\L^*(-\valpha/C) = -C\sum_{i=1}^n l^*(-\alpha_i/C) = \sum_{i=1}^n\alpha_i,
\end{equation}
provided that $\forall i=1,\ldots,n:0\leq\alpha_i\leq C$; otherwise we have $-C\thinspace\L^*(-\valpha/C) = -\infty$.
Hence, for the hinge-loss, we obtain the following pair of primal and dual problem.\\

\noindent \quad Primal:
\begin{equation} 
\inf_{\substack{\vtheta\in\Theta_p\\ \W\in\Hilb}}  \quad \frac{1}{2}\sum_{m=1}^M\frac{\Vert W_m\Vert^2_{Q_m}}{\theta_m} ~+~ C\thinspace\sum_{i=1}^n \max\bracket{0,1-y_i\sum\nolimits_{m=1}^M\big\langle\w_{m\tau(i)},\varphi_m(x_i)\big\rangle} \thinspace
\label{primal_hinge}
\end{equation}%
\newline
\noindent \quad Dual:
\begin{equation}
\inf_{\vtheta\in\Theta_p} ~ \sup_{\zero\preceq\valpha\preceq\C} ~\thinspace -\frac{1}{2}\sum_{i,j=1}^n \alpha_i\alpha_jy_iy_j\sum_{m=1}^M\theta_m\tilde{k}_m(x_i,x_j) ~+~ \sum_{i=1}^n\alpha_i \thinspace.\end{equation}\label{dual_hinge}

%
%
%
%
%

\subsubsection{Single Task Learning}
Starting from the simplest special case, we briefly show how 
single-task learning methods may be recovered from our general framework.
By mapping well understood single-task methods onto our framework,
we hope to achieve two things. First, we believe this will greatly
facilitate understanding for the reader who is familiar with standard
methods like the SVM. Second, we pave the way for applying efficient
training algorithms developed in Section~\ref{sec:op} to these single-task formulations,
for example yielding a new linear solver for 
non-sparse Multiple Kernel Learning as a corollary.
\paragraph{Support Vector Machine} 
In the case of the single-task ($\W=\w$, $Q=1$), single kernel SVM ($M=1$), the primal from Equation~\ref{primal_hinge} and dual from Equation~\ref{dual_hinge} can be greatly simplified:
\begin{equation*} 
\inf_{\w\in\Hilb}  \quad \frac{1}{2}\Vert \w\Vert^2 ~+~ C\thinspace\sum_{i=1}^n \max\bracket{0,1-y_i\big\langle\w,\varphi(x_i)\big\rangle} \thinspace,
\end{equation*}
%
which corresponds to the well-established linear SVM formulation (without bias).
Similarly, the dual is readily obtained from Equation~\ref{dual_hinge}
and is given by
\begin{equation*}
\sup_{\zero\preceq\valpha\preceq\C} ~\thinspace -\frac{1}{2}\sum_{i,j=1}^n \alpha_i\alpha_jy_iy_j k(x_i,x_j) ~+~ \sum_{i=1}^n\alpha_i \thinspace.
\end{equation*}

\paragraph{MKL} $\ell_p$-norm MKL \citep{KloBreSonZie11} is obtained as a special
case of our framework. This case is of particular interest,
as it allows to obtain a linear solver for {$\ell_p$-norm MKL}, as a corollary.
By restricting the number of tasks to one (i.e., $T=1$),
$\W_m$ becomes $\w_m$ and $Q=1$.
Equation~\eqref{primal_hinge} reduces to:
\begin{equation*} 
\inf_{\vtheta\in\Theta_p,\W\in\Hilb}  \quad \frac{1}{2}\sum_{m=1}^M\frac{\Vert \w_m\Vert^2}{\theta_m} ~+~ C\thinspace\sum_{i=1}^n \max\bracket{0,1-y_i\sum\nolimits_{m=1}^M\big\langle\w_{m},\varphi_m(x_i)\big\rangle} \thinspace
\thinspace. \end{equation*}
In agreement with \cite{KloBreSonZieLasMue09}, we recover the dual formulation
from Equation~\ref{dual_hinge}.
\begin{equation*}
\inf_{\vtheta\in\Theta_p} ~ \sup_{\zero\preceq\valpha\preceq\C} ~\thinspace -\frac{1}{2}\sum_{i,j=1}^n \alpha_i\alpha_jy_iy_j\sum_{m=1}^M\theta_m k_m(x_i,x_j) ~+~ \sum_{i=1}^n\alpha_i \thinspace.\end{equation*}
%
%

\subsubsection{Multitask Learning}\label{sec:MTL}

Here, we first derive the primal and dual formulations of regularization-based
multitask learning as a special case of our framework and 
then give an overview of existing variants
that can be mapped onto this formulation as a precursor to novel
instantiations in Section~\ref{extension}.
In this setting, we deal with multiple tasks $t$, but only 
a single kernel or task similarity measure $Q$ (i.e., $M=1$).
The primal thus becomes:
\begin{equation} 
\inf_{\W\in\Hilb}  \quad \frac{1}{2}\Vert W\Vert^2_{Q} ~+~ C\thinspace\sum_{i=1}^n \max\bracket{0,1-y_i\big\langle\w_{\tau(i)},\varphi(x_i)\big\rangle} \thinspace
\thinspace, \label{mtl_primal}\end{equation}
with corresponding dual 
\begin{equation}
\sup_{\zero\preceq\valpha\preceq\C} ~\thinspace -\frac{1}{2}\sum_{i,j=1}^n \alpha_i\alpha_jy_iy_j\tilde{k}(x_i,x_j) ~+~ \sum_{i=1}^n\alpha_i \thinspace,\label{mtl_dual}\end{equation}
where the definition of $\tilde k$ is given in Equation~\ref{eq:def-mtk}.
As we will see in the following, 
the above formulation captures several existing MTL approaches,
which can be expressed by choosing different encodings $Q$
for task similarity. 

\paragraph{Frustratingly Easy Domain Adaptation}\label{frust}

An appealing special case of Graph-regularized MTL was presented by \cite{daum}. They considered the setting of only two tasks (source task and target task), with a fix task relationship. Their \emph{frustratingly easy} idea was to assign a higher similarity to pairs of examples from the same task than between examples from different tasks. 
In a publication titled \emph{Frustratingly Easy Domain Adaptation}, \cite{daum} present a simple, yet appealing special case of graph-regularized MTL. 
They considered the setting of only two tasks (source task and target task), 
with a fix task relationship (i.e., the influence of the two tasks on each other was not determined by their actual similarity). 
Their idea was to assign a higher base-similarity to pairs of examples from the 
same task than between examples from different tasks. This may be expressed by the following multitask kernel:
$$ \tilde k(x,z) = \begin{cases} 2 k(x,z) \quad \tau(x) = \tau(z) \\ k(x,z) \quad \text{  else}\thinspace. \end{cases} $$

%
%
%
From the above, we can readily read off the corresponding 
$Q^{-1}$ (and compute $Q$).
$$Q^{-1} = 
\left( \begin{array}{cc}
2 & 1 \\
1 & 2 \end{array} \right),\quad\quad
Q = 
\left( \begin{array}{cc}
\frac{2}{3} & -\frac{1}{3} \\
-\frac{1}{3} & \frac{2}{3} \end{array} \right).
$$
Given the above, we can express this special case in terms of Equation~\eqref{mtl_primal} and \eqref{mtl_dual}.
With some elementary algebra, 
this method can be viewed as \emph{pulling}
weight vectors of source $\w_s$ and target $\w_t$ towards
a common mean vector $\bar \w$ by means of a regularization term.
If we generalize this idea to allow for multiple cluster
centers, we arrive at \emph{task clustering}, which is
described in the following.


\paragraph{Task Clustering Regularization}\label{cluster}

Here, tasks are grouped into $M$ clusters, whereas parameter vectors
of tasks within each cluster are pulled towards 
the respective cluster center $\bar \w_m = \frac{1}{T_m} \sum_{t=1}^{T_m} \w_t,$ 
where $T_m$ is the number of tasks in cluster $m$ \citep{evgeniou2006learning}.
To understand what $Q$ and $Q^{-1}$ correspond to in terms of
Equations~\ref{mtl_primal} and \ref{mtl_dual}, consider the definition
of the multitask regularizer $\Reg$ for task clustering.
\begin{align}\label{eq:cluster_reg}
   \hspace{-1cm}  R(\w_1, \ldots, \w_T) &= \frac{1}{2}\bracket{\sum_{t=1}^T \lambda \norm{\w_t}^2 +
\sum_{m=1}^M\bracket{\rho\norm{\bar \w_m}^2+\sum_{t=1}^T\rho_m^t\norm{\w_t-\bar\w_m}^2}} \\
   & =\frac{1}{2}\bracket{\sum_{t=1}^T \lambda \norm{\w_t}^2 + \sum_{s,t=1}^TG_{s,t}\langle\w_s,\w_t\rangle} \\ 
   & =\frac{1}{2}\tr\bracket{W(\lambda I + G)W^{\top}}\thinspace, 
\end{align}
%
%
where $M$ is the number of clusters, $\rho_m^t \geq 0$ encodes assignment
of task $t$ to cluster $m$, $\rho$ controls regularization
of cluster centers $\bar \w_m$ and $G$ are given by  
$$ G_{s,t} = \sum_{m=1}^M \left (\rho_m^t \delta_{st} - \frac{\rho_m^s \rho_m^t}{\rho + \sum_{r=1}^T \rho_m^r} \right ).$$
If any task $t$ is assigned to at least one cluster $m$ (i.e., $\forall t \exists m: \rho_m^t > 0$) $G$ is positive definite \citep{evgeniou2006learning}
and we can express the above in terms of our primal formulation
in Equation~\ref{mtl_primal} as $Q=(\lambda I + G)$ and the corresponding dual
as $Q^{-1}=(\lambda I + G)^{-1}$, even for $\lambda = 0$.
We note that the formulation given in Section~\ref{frust} may by expressed via
task clustering regularization, by choosing only one cluster (i.e., $M=1$)
and setting $\lambda = 0$, $\rho = 1$ and $\rho_1^{\text{source}} = \rho_1^{\text{target}} = 1$,
we get $G_{s,t} = \delta_{s,t} - \frac{1}{3}$, equating to the task similarity
matrix $Q$ from the previous section.
%

\paragraph{Graph-regularized MTL}\label{graphreg}

Graph-regularized MTL was established by \cite{evgeniou2006learning} and constitutes
one of the most influential MTL approaches to date. Their method is based on the following
multi-task regularizer, which also forms one of the main inspirations for our framework:
\begin{align}\label{eq:graphreg}
   \hspace{-1cm}  R(\w_1, ..., \w_T) &= \frac{1}{2}\bracket{\sum\nolimits_{t=1}^T\norm{\w_t}^2+\sum\nolimits_{s,t=1}^Ta_{st}\norm{\w_s-\w_t}^2} \\
   & =\frac{1}{2}\bracket{\sum\nolimits_{t=1}^T\norm{\w_t}^2  
         +\sum\nolimits_{s,t=1}^Tl_{s,t}\langle\w_s,\w_t\rangle} \\ 
   & =\frac{1}{2}\tr\bracket{W(I+L)W^{\top}}\thinspace, 
\end{align}
where $A =(a_{st})_{1\leq s,t\leq T}\in \mathbb{R}^{T\times T}$ is a given graph adjacency matrix encoding the pairwise similarities of the tasks,
$L = D - A$ denotes the corresponding graph Laplacian, where $D_{i,j}:= \delta_{i,j} \sum_k A_{i,k}$, and $I$ is a ${T\times T}$ identity matrix.
Note that the number of zero eigenvalues of the graph Laplacian corresponds to the number
of connected components.
We may view graph-regularized MTL as an instantiation of our general primal problem, Problem~\ref{prob:primal},
where we have only one task similarity measure $Q_1 = I + L$ (i.e., $M=1$).
As the graph Laplacian $L$ is not invertible in general, 
we use its pseudo-inverse $L^\dagger$ to express
the dual formulation of the above MTL regularizer.
\begin{align}
Q_{s,t}^{-1} = L_{s,t}^\dagger= \sum_{i=1}^r \sigma_i \v_{is}^T\v_{it},
\label{pseudo-inv}
\end{align}
where $r$ is the rank of $L$, $\sigma_i$ are the eigenvalues of $L$
and $V = (\v_{s,t})$ is the orthogonal matrix of eigenvectors.
%
%

%

\paragraph{Multi-task Kernels}\label{mtl_kernels}

In contrast to graph-regularized MTL, where task relations are captured
by an adjacency matrix or graph Laplacian as discussed in the previous paragraph,
task relationships may directly be expressed in terms of a kernel on tasks $K_{\textnormal{tasks}}$.
This relationship has been illuminated in Section~\ref{sec:rel-mkl},
where we have seen that the kernel on tasks corresponds to $Q^{-1}$ in our dual MTL formulation.
A formulation involving a combination of several MTL kernels with a fix
weighting was explored by \cite{Jacob2008}
in the context of Bioinformatics.
In its most basic form, the authors considered a multitask kernel of the form $$K((x, t), (z, s)) = K_{\textnormal{base}}(x,z) \cdot K_{\textnormal{tasks}}(t,s).$$
Furthermore, the authors considered a sum of different multi-task kernels, among them the corner cases $K_{\textnormal{Dirac}}(t,s) = \delta_{s,t}$ (independent tasks)
and the uniform kernel $K_{\textnormal{Uni}}(t,s) = 1$ (uniformly related tasks). 
In general, their dual formulation is given by
$$K((x, t), (z, s)) = \sum_{m=1}^M K_{\textnormal{base}}(x,z) \cdot K_{\textnormal{tasks}}^{(m)}(t,s).$$

The above is a very interesting special case and can easily be expressed
within our general framework. For this, consider the dual formulation
given in Equation~\ref{dual_hinge} for $Q^{(m)-1} = K_{\textnormal{tasks}}^{(m)}$
and $\theta_1 = \ldots = \theta_M = 1$. In other words, the above also
constitutes a form of multitask multiple kernel learning, however, without
actually learning the kernel weights $\Theta_m$. Nevertheless, the choice and discussion
of different multitask kernels $K_{\textnormal{tasks}}^{(m)}$ in
\cite{Jacob2008} is of high relevance with respect to the family of methods
explored in this work.

\subsection{Proposing Novel Instances of Multi-task Learning Machines}\label{extension}

We now move ahead and derive novel instantiations from our general framework. 
Most importantly, we go beyond previous formulations 
by learning or refining task similarities from data
using MKL \emph{as an engine}.

\begin{figure}[h!]
  \centering
  \subfigure[Multigraph MT-MKL]{
    \label{new_mutigraph}
    \includegraphics[width=0.90\textwidth]{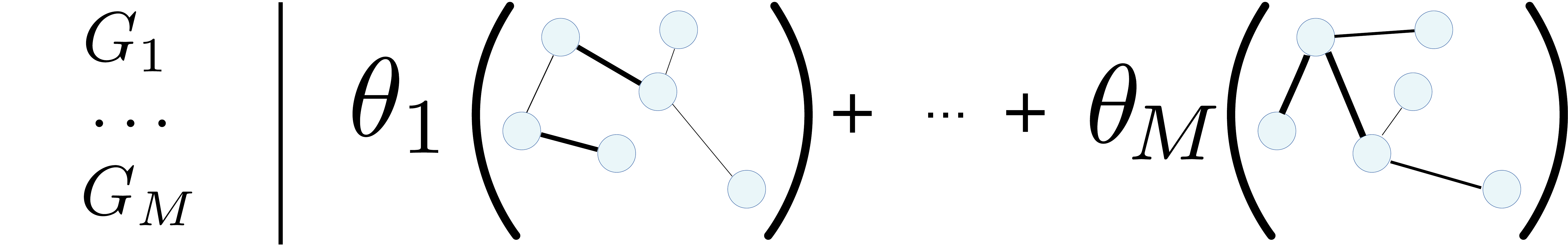}
  }
  \subfigure[Hierarchical MT-MKL]{
    \label{new_hmtmkl}
    \includegraphics[width=0.90\textwidth]{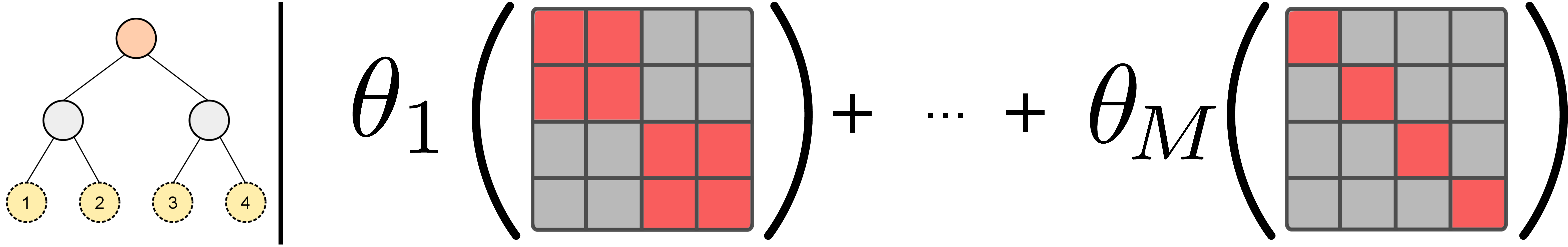}
  }  
  \subfigure[Smooth MT-MKL]{%
    \label{new_smooth}
    \includegraphics[width=0.90\textwidth]{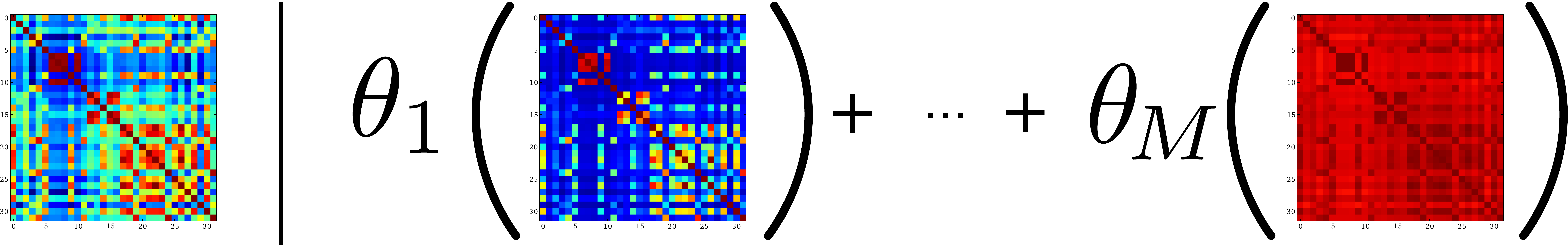}
  }  
  \caption{Learning additive transformations of task similarities: (a)
    Multigraph MT-MKL where one combines similarities from multiple
    independent graphs (which includes the approaches proposed in
    \citet{Widmer2010,Jacob2008}); (b) Hierarchical MT-MKL where one
    uses a tree to generate specific similarity matrices (as proposed
    in \citet{widmer2010leveraging,Widmer2010,Goe11,Widmer2012}); and
    (c) Smooth MT-MKL where one uses multiple transformations of an
    existing similarity matrix for linear
    combination. \label{new_mtmkl}}
\end{figure}

\subsubsection{Multi-graph MT-MKL}

One of the most popular MTL approaches is graph-regularized MTL
by \cite{Evgeniou2004}. We have seen in Section~\ref{sec:MTL},
that such a graph is expressed as a adjacency matrix $A$ and
may alternatively be expressed in terms of its graph Laplacian $L$.
Our extension readily deals with multiple graphs
encoding task similarity $A_m =(a_{mst})_{1\leq s,t\leq T}\in \mathbb{R}^{T\times T}$, which is of interest in cases where
- as in Multiple kernel learning - we have access to alternative
sources of task similarity and it is unclear which one is best suited.
This concept gives rise to the \emph{multi-graph MTL} regularizer
$$ R(\W)= \frac{1}{2}\tr\bracket{\sum\nolimits_{m=1}^MW_m(I+L_m)W_m^{\top}} \thinspace,$$
where $L_m$ denotes the graph Laplacian corresponding to $A_m$. 
As before, we learn a weighting of the given graphs, therefore
determining which measures are best suited to maximize
prediction accuracy.

\subsubsection{Hierarchical MT-MKL}\label{hmtmkl}

Recall that in task clustering, parameter vectors of tasks within the
same cluster are coupled (Equation~\ref{eq:cluster_reg}). The
\emph{strength} of that coupling, however, has be be chosen in advance
and remains fixed throughout the learning procedure. We extend the
formulation of task clustering by introducing a weighting $\theta_m$
to task cluster $m$ and tuning this weighting using our framework.
%
We decompose $G$ over clusters and arrive at the following MTL regularizer
\begin{align}\label{eq:weighted_cluster_reg}
   \hspace{-1cm}  R(\w_1, \ldots, \w_T) & =\frac{1}{2}\bracket{\sum\nolimits_{m=1}^M\norm{\w_m}^2  
         +\sum\nolimits_{m=1}^M \theta_m \sum\nolimits_{s,t=1}^TG_{s,t}^m\langle\w_s,\w_t\rangle} \\ 
   & =\frac{1}{2}\sum\nolimits_{m=1}^M\tr\bracket{ \theta_m W(I+G^m)W^{\top}}\thinspace, 
\end{align}
%
%
where $G^m$ is given by  
$$ G_{s,t}^m = \rho_m^t \delta_{st} - \frac{\rho_m^s \rho_m^t}{\rho + \sum_{r=1}^T \rho_m^r}.$$
Note that, if not all tasks belong to the same cluster, 
$G^m$ will not be invertible. Therefore, we need to 
express the mapping onto the dual of our general framework from Equation~\ref{dual_hinge}
in terms of the pseudo-inverse (see Equation~\ref{pseudo-inv})
of $G_m$: $Q_m^{-1} = G_m^{\dagger}$.

An important special case of the above is given
by a scenario where task relationships are described by a 
hierarchical structure $\mathcal G$ (see Figure~\ref{new_hmtmkl}), such as a tree or a directed acyclic graph.
Assuming hierarchical relations between tasks is particularly relevant to
Computational Biology where often different tasks correspond to different
organisms. In this context, we expect that the longer the common evolutionary
history between two organisms, the more beneficial it is to share information
between these organisms in a MTL setting.
The tasks correspond to the leaves or terminal nodes
and each inner node $n_m$ defines a cluster $m$,
by grouping tasks of all terminal nodes
that are descendants of the current node $n_m$.
As before, task clusters $G$ can be used in the way discussed
in the previous section.

\subsubsection{Smooth hierarchical MT-MKL}\label{smooth_MTMKL}
Finally, we present a variant that may be regarded as a smooth version of the
hierarchical MT-MKL approach presented above. Here, however, we require access
to a given task similarity matrix, which is then subsequently transformed
by squared exponentials with different length scales, for instance, $\mathbf Q^{(m)}_{st} = exp(A_{st} / \sigma_m)$.
We use MT-MKL to learn a weighting of the kernels associated with the different
length scales, which corresponds to finding the right level in the hierarchy to
trade off information between tasks.
As an example, consider Figure~\ref{new_smooth}, where we show the original
task similarity matrix and the transformed matrices at different length scales.
\section{Algorithms}\label{sec:op}

In this section, we present efficient optimization algorithms to solve the primal and dual problems, i.e., Problems~\ref{prob:primal}~and~\ref{prob:dual},
respectively. We distinguish the cases of linear and non-linear kernel matrices. 
For non-linear kernels, we can simply use existing MKL implementations, while, for linear kernels,
we develop a specifically tailored large-scale algorithm that allows us to 
train on problems with a large number of data points and dimensions, as demonstrated on several data sets.
We can even employ this algorithm for non-linear kernels,
if the kernel admits a sparse, efficiently computable feature representation. For example, this is the case 
for certain string kernels and polynomial kernels of degree 2 or 3. 
Our algorithms are embedded into the COFFIN framework \citep{coffin} and integrated  
into the SHOGUN large-scale machine learning toolbox \citep{sonnenburg2010shogun}.

\subsection{General Algorithms for Non-linear Kernels}\label{sec:wstep-mtkernel}

A very convenient way to numerically solve the proposed framework is to simply exploit existing MKL implementations.
To see this, recall from Section~\ref{sec:rel-mkl} that if we use the multi-task kernels $\tilde{k}_1,\ldots,\tilde{k}_M$ 
as defined in \eqref{eq:def-mtk} as the set of multiple kernels, 
the completely dualized MKL formulation (see Problem~\ref{prob:complete_dual}) is given by,
$$ \inf_{\vtheta\in\Theta_p} ~ \sup_{\valpha\in\R^n:\sum_{i=1}^n\alpha_iy_i=0} ~\thinspace -\frac{1}{2}\Bigg\Vert\Bigg(\sum_{i,j=1}^n \alpha_i\alpha_jy_iy_j\sum_{m=1}^M\theta_m\tilde{k}_m(x_i,x_j)\Bigg)_{1\leq m\leq M}\Bigg\Vert_{p^*} 
      ~-~ C\thinspace \L^*(-\valpha/C)  \thinspace.$$
%
An efficient optimization approach is by \cite{VarmaSMO}, who optimize the completely dualized MKL formulation.
This implementation comes along without a $\vtheta$-step, but any of the $\alpha_i$-steps 
computations of the $\alpha_i$-steps are  more costly as in the case of vanilla (MT-)SVMs.

Further, combining the partially dualized formulation in Problem~\ref{prob:dual} 
with the definition of multi-task kernels from \eqref{eq:def-mtk},
we arrive at an equivalent problem to \eqref{eq:prob-mtk}, that is,
$$ \inf_{\vtheta\in\Theta_p} ~ \sup_{\valpha\in\R^n} ~\thinspace -\frac{1}{2}\sum_{i,j=1}^n \alpha_i\alpha_jy_iy_j\sum_{m=1}^M\tilde{k}_m(x_i,x_j) 
      ~-~ C\thinspace \L^*(-\valpha/C)  \thinspace,$$
which is exactly the optimization problem of $\ell_p$-norm multiple kernel learning as described in \cite{KloBreSonZie11}.
We may thus build on existing research in the field of MKL and use one of the prevalent efficient implementations to solve $\ell_p$-norm MKL\@.
Most of the $\ell_p$-norm MKL solvers are specifically tailored to the hinge loss. Proven implementations are, for example, 
the interleaved optimization method of \cite{KloBreSonZie11},
which is directly integrated into the SVMLight module \citep{Joa99} of the SHOGUN toolbox 
such that the $\vtheta$-step is performed after each decomposition step, i.e., after solving the small QP occurring in SVMLight,
which allows very fast convergence \citep{SonRaeSchSch06}. 

For an overview of MKL algorithms and their implementations, see the survey paper by \cite{GonenA11}.

\subsection{A Large-scale Algorithm for Linear or String Kernels and Beyond}

For specific kernels such as linear kernels and string kernels---and, more generally, any kernel admitting an efficient
feature space representation---, we can derive a specifically tailored large-scale algorithm.
This requires considerably more work than the algorithm presented in the previous subsection.

\subsubsection{Overview}

From a top-level view, the upcoming algorithm underlies the core idea of alternating the following two steps:
\begin{enumerate}
    \item the $\vtheta$ step, where the kernel weights are improved
    \item the $\W$ step, where the remaining primal variables are improved.
\end{enumerate}
\begin{algorithm}[!h]
  \begin{algorithmic}[1]
  \small
  \STATE \textbf{input:} data $x_1,\ldots,x_n\in\X$ and labels ~$y_1,\ldots,y_n\in\{-1,1\}$ associated with tasks ~$\tau(1),\ldots,\tau(n)\in\{1,\ldots,T\}$; 
        feature vectors $\phi_1(x_i),\ldots,\phi_M(x_i)$; task similarity matrices $Q_1,\ldots,Q_M$; optimization precision $\varepsilon$
  \STATE initialize $\theta_m:=\sqrt[p]{1/M}$ for all $m=1,\ldots,M$, initialize $\W=\zero$
  \STATE \textbf{while} optimality conditions are not satisfied within tolerance $\epsilon$ {\bf do} \\[1pt]
    \STATE \qquad \emph{$\W$ descent step:} compute new $\W$ such that the obj. $\Reg_{\vtheta}(\W)  \thinspace+\thinspace C\thinspace\L(\W)$ decreases \label{line:W_step} 
    \STATE       \qquad\qquad\quad~~\thinspace\qquad\qquad $\W :=\argmin_{\widetilde{\W}}~\Reg_{\vtheta}(\widetilde{\W})  \thinspace+\thinspace C\thinspace\L(\widetilde{\W})$  \label{line:W_step_exact} \\[2pt]
    \STATE \qquad \emph{$\vtheta$ step:} compute minimizer $\vtheta:=\argmin_{\widetilde{\vtheta}\in\Theta_p}~\Reg_{\widetilde{\vtheta}}(\W)  \thinspace+\thinspace C\thinspace\L(\W)$  according to \eqref{eq:theta_update} \label{line:theta_step}
  \STATE \textbf{end while}
  \STATE \textbf{output:} $\epsilon$-accurate optimal hypothesis $\W$ and kernel weights $\vtheta$
  \end{algorithmic}
  \caption{\label{alg:blueprint} \textsc{(Blueprint of the large-scale optimization algorithm).}  The MKL module ($\theta$ step) is wrapped around the MTL module ($\W$ step).}
\end{algorithm}
These steps are illustrated in Algorithm Table~\ref{alg:blueprint}.
We observe from the table that the variables are split into the two
sets $\set{\theta_m|m=1,\ldots,M}$ and
$\set{\w_{mt}|m=1,\ldots,M,t=1,\ldots,T}$.  The algorithm then
alternatingly optimizes with respect to one or the other set until the
optimality conditions are approximately satisfied.  We will analyze
convergence of this optimization scheme later in this section.  Note
that similar algorithms have been used in the context of the group
lasso and multiple kernel learning by, for instance, \cite{RotFis08},
\cite{Xuetal10}, and \cite{KloBreSonZie11}.

\subsubsection{Solving the $\vtheta$ Step}

In this section, we discuss how to compute the update of the kernel weights $\vtheta$ as carried out in Line~\ref{line:theta_step} of Algorithm~\ref{alg:blueprint}.
Note that for fixed $\W\in\Hilb$ it holds
$$ \arginf_{\vtheta\in\Theta_p} \thinspace~ \Reg_{\vtheta}(\W)  \thinspace+\thinspace C\thinspace\L\big(A(\W)\big)  
    ~\thinspace=~\thinspace  \arginf_{\vtheta\in\Theta_p} \thinspace~ \Reg_{\vtheta}(\W) \thinspace, $$
where $\Reg_{\vtheta}(\W)  \thinspace=\thinspace \frac{1}{2}\sum_{m=1}^M\frac{\tr(W_mQ_mW_m)}{\theta_m} $. Furthermore, by Lagrangian duality,
\begin{align*}
  \hspace{0.5cm}& \hspace{-0.5cm} \inf_{\vtheta\in\Theta_p}  ~ \frac{1}{2}\sum_{m=1}^M\frac{\tr(W_mQ_mW_m)}{\theta_m}  
      \quad =\quad \max_{\lambda\geq0} ~~ \inf_{\vtheta\succeq\zero}  ~ \frac{1}{2}\sum_{m=1}^M\frac{\tr(W_mQ_mW_m)}{\theta_m} \thinspace+\thinspace \lambda\sum_{m=1}^M \theta_m^p \\
    &=~ \inf_{\vtheta\succeq\zero}  ~ \underbrace{\frac{1}{2}\sum_{m=1}^M\frac{\tr(W_mQ_mW_m)}{\theta_m} \thinspace+\thinspace \lambda^*\sum_{m=1}^M \theta_m^p}_{=:\psi(\vtheta)} 
                \thinspace,
\end{align*}
where we denote the optimal $\lambda$ in the above maximization by $\lambda^\star$. The infimum is either attained at the boundary of the constraints or when 
$\nabla_{\vtheta} \psi(\vtheta) = 0$,  thus the optimal point $\vtheta^\star$ satisfies $\theta_m^\star=\bracket{\tr(W_mQ_mW_m)/\lambda^\star}^{1/(p+1)}$ for any $m=1,\ldots,M$.
Because $\vtheta^\star\in\Theta_p$, i.e., $\norm{\vtheta}_p=1$, it follows $\lambda^\star=\bracket{\sum_{m=1}^M \tr(W_mQ_mW_m)^{p/(p+1)}}^{(p+1)/p}$, under the
minimal assumption that $\W\neq\zero$. Thus, because
$\tr(W_mQ_mW_m)=\sum_{s,t=1}^Tq_{mst}\inner{\w_{ms},\w_{mt}}$,
\begin{equation}\label{eq:theta_update}
    \forall m=1,\ldots,M: \quad \theta_m^\star=\frac{\sqrt[p+1]{\sum_{s,t=1}^Tq_{mst}\inner{\w_{ms},\w_{mt}}}}{\bracket{\sum_{m=1}^M\hspace{-0.22em}\sqrt[p+1]{\sum_{s,t=1}^Tq_{mst}\inner{\w_{ms},\w_{mt}}}^{\thinspace p}}^{1/p}} \thinspace.
\end{equation}

\subsubsection{Solving the $\W$ Descent Step}

To solve the $W$ step as carried out in Line~\ref{line:W_step} of Algorithm~\ref{alg:blueprint}, 
we consider the kernel weights $\set{\theta_m|m=1,\ldots,M}$ as being fixed and optimize solely with respect to $\W$.
In fact, we perform the $\W$ descent step in the dual, i.e., by optimizing the dual objective of 
Problem~\ref{prob:dual}, i.e., solving
$$ \sup_{\valpha\in\Real^n}  ~-\Reg_{\vtheta}^*(A^*(\valpha)) \thinspace-\thinspace C\thinspace \L^*(-(\valpha)/C) \thinspace.$$
Although our framework is also valid for other loss functions, for the presentation of the algorithm, 
we make a specific choice of a proven loss function, that is, the hinge loss $l(a)=\max(0,1-a)$, so that by \eqref{eq:dualloss_hinge}, the above task becomes
\begin{equation}\label{eq:algotask}
  \sup_{\valpha\in\Real^n: \zero\preceq\valpha\preceq\C}  ~-\Reg_{\vtheta}^*(A^*(\valpha)) \thinspace+\thinspace \sum_{i=1}^n \alpha_i \thinspace. 
\end{equation}
Our algorithm optimizes \eqref{eq:algotask} by dual coordinate ascent, i.e., by optimizing the dual variables $\alpha_i$ one after another 
(i.e., only a single dual variable $\alpha_i$ is optimized at a time),
$$ \sup_{d\in\Real:\thinspace0\leq\alpha_i+d\leq C}  ~~-\Reg_{\vtheta}^*(A^*(\valpha+d\e_i))  \thinspace+\thinspace \sum_{i=1}^n \alpha_i +d\thinspace,$$
where we denote the unit vector of $i$th coordinate in $\Real^n$ by $\e_i$.
As we will see, this task can be performed analytically; 
however, performed purely in the dual involves computing a sum over all support vectors which is infeasible for large $n$. 
Our proposed algorithm is, instead, based on the application of the representer theorem carried out in Section~\ref{sec:representer}:
recall from \eqref{eq:representer} that, for all $m=1,\ldots,M$ and $t=1,\ldots,T$, it holds 
$$\w_{mt} = \theta_m\sum_{i=1}^nq^{(-1)}_{m\tau(i)t}\alpha_iy_i\varphi_m(x_i) \thinspace.$$
The core idea is to express the update of the $\alpha_i$ in the coordinate ascent procedure solely in terms
of the vectors $\w_{mt}$. While optimizing the variables $\alpha_i$ one after another, we keep track of the 
changes in the vectors $\w_{mt}$. This procedure is reminiscent of the \emph{dual coordinate ascent} method,
but differs in the way the objective is computed.
Of course, this implies that we need to manipulate feature vectors, which explains why our approach relies on efficient infrastructure of 
storing and computing feature vectors and their inner products. If the infrastructure is adequate so that computing inner products in the feature space 
is more efficient than computing a row of the kernel matrix, our algorithm will have a substantial gain.

\paragraph{Expressing the update of a single variable $\alpha_i$ in terms of the vectors $\w_{mt}$}
As argued above, our aim is to express the (analytical) computation of
$$ \sup_{d\in\Real:\thinspace0\leq\alpha_i+d\leq C}  ~~-\Reg_{\vtheta}^*(A^*(\valpha+d\e_i))  \thinspace+\thinspace \sum_{i=1}^n \alpha_i+d \thinspace.$$
solely in terms of  the vectors $\w_{mt}$. To start the derivation, note that, by \eqref{eq:part_dual_constr},
$$ \Reg_{\vtheta}^*(A^*(\valpha+d\e_i)) ~=~ \frac{1}{2}\sum_{m=1}^M\theta_m\norm{A^*(\valpha+d\e_i)}^2_{Q_m^{-1}} $$
with, by \eqref{eq:mtk-norm-term}, 
$$ \norm{A^*(\valpha+d\e_i)}^2_{Q_m^{-1}} ~=~ \sum_{j,\tilde{j}=1}^n\alpha_j\alpha_{\tilde{j}}y_jy_{\tilde{j}}\tilde{k}_m(x_j,x_{\tilde{j}})
   \thinspace+\thinspace 2dy_i\sum_{j=1}^n\alpha_jy_j\tilde{k}_m(x_i,x_j) \thinspace+\thinspace d^2k_m(x_i,x_i) \thinspace, $$
where 
$$\tilde{k}_m(x_i,x_j)\thinspace=\thinspace q^{(-1)}_{m\tau(i)\tau(j)}k_m(x_i,x_j)$$
is the $m$th multi-task kernel as defined in \eqref{eq:def-mtk}. Thus,
\begin{eqnarray*}
  && \argsup_{d\in\Real:\thinspace0\leq\alpha_i+d\leq C}  ~-\Reg_{\vtheta}^*(A^*(\valpha+d\e_i))  \thinspace+\thinspace \sum_{i=1}^n \alpha_i+d  \\
  &=& \argsup_{d\in\Real:\thinspace0\leq\alpha_i+d\leq C} ~~  d\thinspace-\thinspace 
            dy_i\sum_{j=1}^n\alpha_jy_j\bracket{\sum\nolimits_{m=1}^M\theta_m\tilde{k}_m(x_i,x_j)} \thinspace-\thinspace \frac{1}{2}d^2\bracket{\sum\nolimits_{m=1}^M\theta_m\tilde{k}_m(x_i,x_i)} \\
  &=& \argsup_{d\in\Real:\thinspace0\leq\alpha_i+d\leq C} ~~ d\thinspace-\thinspace \underbrace{dy_i\bracket{\sum\nolimits_{m=1}^M\inner{\w_{m\tau(i)},\varphi_m(x_i)}} \thinspace-\thinspace \frac{1}{2}d^2\bracket{\sum\nolimits_{m=1}^M\theta_m\tilde{k}_m(x_i,x_i)}}_{=:\psi(d)} \thinspace. 
\end{eqnarray*}
The optimum of $\psi(d)$ is either attained at the boundaries of the
constraint $0\leq\alpha_i+d\leq C$ or when $\psi'(d)=0$.  Hence, the
optimal $d^\star$ can be expressed analytically as
\begin{equation}\label{eq:update}
  d^\star=\max\bracket{-\alpha_i\thinspace,~\min\bracket{C-\alpha_i\thinspace,\medspace\frac{1-y_i\sum\nolimits_{m=1}^M\theta_m\inner{\w_{m\tau(i)},\varphi_m(x_i)}}{\sum\nolimits_{m=1}^M\theta_m\tilde{k}_m(x_i,x_i)}}} \thinspace. 
\end{equation}
Whenever we update an $\alpha_i$ according to
$$ \alpha_i^{\rm new} := \alpha_i^{\rm old} + d^\star $$
with $d$ computed as in \eqref{eq:update}, we need to also update the vectors $\w_{mt}$, 
$m=1,\ldots,M$, $t=1,\ldots,T$,  according to 
\begin{equation}\label{eq:algo_update1}
  \w_{mt}^{\rm new} := \w_{mt}^{\rm old}+d\theta_mq^{(-1)}_{m\tau(i)t}y_i\varphi_m(x_i) \thinspace,
\end{equation}
to be consistent with \eqref{eq:representer}.  Similarly, we need to
update the vectors $\w_{mt}$ after each $\theta$ step according to
\begin{equation}\label{eq:algo_update2}
  \w_{mt}^{\rm new} := \bracket{\theta_m^{\rm new}/\theta_m^{\rm old}}\w_{mt}^{\rm old} \thinspace . 
\end{equation}
To avoid recurrences in the iterates, a $\vtheta$-step should only be
performed if the primal objective has decreased between subsequent
$\vtheta$-steps.  Thus, after each $\valpha$ epoch, the primal
objective needs to be computed in terms of $\W$. As described above,
the algorithm keeps $\W$ up to date when $\valpha$ changes, which
makes this task particular simple.

The resulting large-scale algorithm is summarized in
Algorithm~Table~\ref{alg:mtl-libsvm}.  Data and the labels are input
to the algorithm as well as a sub-procedure for efficient computation
of feature maps (cf.\ Section~\ref{sec:coffin}).  Lines 2 and 3
initialize the optimization variables. In Line~4 the inverses of the
task similarity matrices are pre-computed.
Algorithm~\ref{alg:mtl-libsvm} iterates over Lines 7--16 until the
stopping criterion falls under a pre-defined accuracy threshold
$\varepsilon$.  In Lines 7--11 the line search is computed for all
dual variables. Lines 14 and 15 update the primal variables and kernel
weights to be consistent with the representer theorem, only if the
primal objective has decreased since the last $\vtheta$-step.  We stop
Algorithm~\ref{alg:mtl-libsvm} when the relative change in the
objective $o$ is less than $\epsilon$. Notice that we do not optimize
the $W$ step to full precision, but instead alternate between one pass
over the $\alpha_i$ and a $\vtheta$ step.
\begin{algorithm}[!h]
  \begin{algorithmic}[1]
  \small
  \STATE \textbf{input:} data $x_1,\ldots,x_n\in\X$ and labels ~$y_1,\ldots,y_n\in\{-1,1\}$ associated with tasks ~$\tau(1),\ldots,\tau(n)\in\{1,\ldots,T\}$; efficiently computable feature maps $\varphi_1,\ldots,\varphi_M$; task similarity matrices $Q_1,\ldots,Q_M$; optimization precision $\varepsilon$
  \STATE for all $i\in\{1,\ldots,n\}$ initialize $\alpha_i=0$ 
    \STATE for all $m\in\{1,\ldots,M\}$ and $t\in\{1,\ldots,T\}$, initialize $\w_{mt}$ according to \eqref{eq:representer} 
    \STATE for all $m\in\{1,\ldots,M\}$, compute inverse $Q_m^{-1}=\big(q_{mst}^{(-1)}\big)_{1\leq s,t\leq T}$
    \STATE initialize primal objective $o=nC$
  \STATE \textbf{while} optimality conditions are not satisfied {\bf do}
  \STATE \qquad \textbf{for} all $i\in\{1,\ldots,n\}$ 
  \STATE \qquad\qquad compute $d$ according to \eqref{eq:update}
    \STATE \qquad\qquad update $\alpha_i:=\alpha_i+d$ 
  \STATE \qquad\qquad for all $m\in\{1,\ldots,M\}$ and $t\in\{1,\ldots,T\}$, update $\w_{mt}$ according to \eqref{eq:algo_update1}
  \STATE \qquad \textbf{end for}
    \STATE \qquad store primal objective $o^{\rm old}=o$ and compute new primal objective $o$
    \STATE \qquad \textbf{if} primal objective has decreased, i.e., $o<o^{\rm old}$
    \STATE \qquad\qquad for all $m\in\{1,\ldots,M\}$, compute $\theta_m$ from $\w_{m1},\ldots,\w_{mT}$ according to \eqref{eq:theta_update}
    \STATE \qquad\qquad for all $m\in\{1,\ldots,M\}$ and $t\in\{1,\ldots,T\}$, update $\w_{mt}$ according to \eqref{eq:algo_update2}
    \STATE \qquad \textbf{end if}
  \STATE \textbf{end while}
  \STATE \textbf{output:} $\epsilon$-accurate optimal hypothesis $\W=(\w_{mt})_{1\leq m\leq M,1\leq t\leq T}$ and kernel weights $\vtheta=(\theta_m)_{1\leq m\leq M}$
  \end{algorithmic}
  \caption{\label{alg:mtl-libsvm} \textsc{(Dual-coordinate-ascent-based MTL training algorithm).}  Generalization of the LibLinear training algorithm to multiple tasks and multiple linear kernels.}
\end{algorithm}

\subsubsection{Details on the Implementation}\label{sec:coffin}

We have implemented the optimization algorithms described in the
previous section into the general framework of the SHOGUN machine
learning toolbox \citep{sonnenburg2010shogun}.  Besides the described
implementations for binary classification, we also provide
implementations for novelty detection and regression.  Furthermore,
the user may choose an optimization scheme, that is, decide whether
one of the classic, non-linear MKL solvers shall be used (either the
analytic optimization algorithm of \cite{KloBreSonZie11}, the cutting
plane method of \cite{SonRaeSchSch06}, or the Newton algorithm by
\cite{KloBreSonZieLasMue09}), or the novel implementation for
efficiently computable feature maps.  Our implementation can be
downloaded from \url{http://www.shogun-toolbox.org}.

In the more conventional family of approaches, the \emph{wrapper
  algorithms}, an optimization scheme on $\vtheta$ wraps around a
conventional SVM solver (for instance, LIBSVM and SVMLIGHT are
integrated into SHOGUN) using a single multi-task kernel.
Effectively, this results in alternatingly solving for $\valpha$ and
$\vtheta$.  For the $\vtheta$-step, SHOGUN offers the three choices
listed above.  The second, much faster approach performs interleaved
optimization and thus requires modification of the core SVM
optimization algorithm. This is currently either integrated into the
chunking-based SVRlight and SVMlight module. Lastly, the completely
new optimization scheme as described in
Algorithm~Table~\ref{alg:mtl-libsvm} is implemented and connected with
the module for computing the $\vtheta$-step.

Note that the implementations for non-linear kernels come with the
option of either pre-computing the kernel or computing the kernel on
the fly for large-scale data sets. For truly large-scale MT-MKL, a
linear or string kernel should be used.  This is implemented as an
internal interface the COFFIN module of SHOGUN \citep{coffin}.

\subsection{Convergence Analysis}\label{sec:converg}

In this section, we establish convergence of
Algorithm~\ref{alg:blueprint} under mild assumptions.  To this end, we
build on the existing theory of convergence of the block coordinate
descent method.  Classical results usually assume that the function to
be optimized is strictly convex and continuously differentiable.  This
assertion is frequently violated in machine learning when, for
instance, the hinge loss is employed.  In contrast, we base our
convergence analysis on the work of \cite{Tseng2001} concerning the
convergence of the block coordinate descent method.  The following
proposition is a direct consequence of Lemma 3.1 and Theorem 4.1 in
\cite{Tseng2001}.

\begin{proposition}\label{prop:tseng}
  Let $f:\mathbb R^{d_1+\cdots+d_R}\rightarrow\mathbb R\cup\{\infty\}$ be a function.   Put $d=d_1+\cdots+d_R$.
    Suppose that $f$ can be decomposed into $f(\a_1,\ldots,\a_r) = f_0(\a_1,\ldots,\a_r) + \sum_{r=1}^Rf_r(\a_r)$
    for some $f_0:\mathbb R^d\rightarrow\mathbb R\cup\{\infty\}$ and $f_r:\mathbb R^{d_r}\rightarrow\mathbb R\cup\{\infty\}$, 
  $r=1,\ldots,R$.
  Initialize the block coordinate descent method by $\a^0=(\a^0_1,\ldots,\a^0_R)$. Let $(r_k)_{k\in\mathbb N}\subset\{1,\ldots,R\}$ be a sequence of
    coordinate blocks. Define the iterates $\a^k=(\a^k_1,\ldots,\a^k_R)$, $k>0$, by
    \begin{equation}\label{eq:sequence}
      \a_{r_k}^{k+1}\in\argmin_{\mathfrak A \in\mathbb R^{d_{r_k}}} f\big(\a_1^{k+1},\cdots,\a_{r_k-1}^{k+1},\mathfrak A,\a_{r_k+1}^{k},
         \cdots,\a_{R}^{k}\big) \thinspace, ~~ \a_{r}^{k+1}:=\a_{r}^k\thinspace, ~~ r\neq r_k\thinspace, ~~ k\in\mathbb N_0\thinspace. 
  \end{equation}
  Assume that 
    \begin{itemize}
      \item[(A1)] $f$ is convex and proper (i.e., $f\not\equiv\infty$)
      \item[(A2)] the sublevel set $\mathcal A^0:=\{\a\in\mathbb R^d: f(\a)\leq f(\a^0)\}$ is compact and $f$ is continuous on $\mathcal A^0$ \\
        (\textsc{assures existence of minimizer in \eqref{eq:sequence}})
        \item[(A3)] $\dom(f_0):=\{a\in\mathbb R^d:f_0(\a)<\infty\}$ is open and $f_0$ is G\^ateaux differentiable (for instance, \emph{continuously differentiable}) on $\dom(f_0)$ \\
        (\textsc{yields regularity---i.e., any coordinate-wise minimum is a minimum of $f$})
        \item[(A4)] it exists a number $T\in\mathbb N$ so that, for each $k\in\mathbb N$ and $r\in\{1,\ldots,R\}$, there is $\tilde{k}\in\{k,\ldots,k+T\}$ with $r_{\tilde{k}}=r$\thinspace. \\
        (\textsc{ensures that each coordinate block is optimized ``sufficiently often''})
    \end{itemize}
  Then the minimizer in \eqref{eq:sequence} exists and any cluster point of the sequence $(\a^k)_{k\in\mathbb N}$ minimizes $f$ over $\mathcal A$.
\end{proposition}

\begin{corollary}
  Assume that
    \begin{itemize}[label=\wackyenum*]
        \item[(B1)] the data is represented by $\phi_m(x_i)\in\R^{e_m}$, $i=1,\ldots,n$, $e_m<\infty$, $m=1,\ldots,M$.
    \item[(B2)] the loss function $l$ is convex, finite in 0, and continuous on its domain $\dom(l)$
        \item[(B3)] the task similarity matrices $Q_1,\ldots,Q_T$ are positive definite
      \item[(B3)] any iterate $\vtheta=(\theta_1,\ldots,\theta_M)$ traversed by Algorithm~\ref{alg:blueprint} has $\theta_m>0$, $m=1,\ldots,M$
        \item[(B4)] the exact search specified in Line~\ref{line:W_step_exact} of Algorithm~\ref{alg:blueprint} is performed
    \end{itemize}
    Then Algorithm~\ref{alg:blueprint} is well-defined and any cluster point of the sequence traversed by the Algorithm~\ref{alg:blueprint} 
    is a minimal point of Problem~\ref{prob:primal}.
\end{corollary}

\begin{proof}
  The corollary is obtained by applying Proposition~\ref{prop:tseng} to Problem~\ref{prob:primal}, that is,
    \begin{equation}\label{eq:prob1} 
      \inf_{\W,\thinspace\vtheta:~\vtheta\in\Theta_p} ~\frac{1}{2}\sum_{m=1}^M\frac{\Vert W_m\Vert^2_{Q_m}}{\theta_m} 
          \thinspace+\thinspace C\sum_{i=1}^nl\bigg(y_i\sum_{m=1}^M\big\langle\w_{m\tau(i)},\varphi_m(x_i)\big\rangle\bigg) \thinspace, 
    \end{equation}
    where $\Theta_p=\{\vtheta\in\mathbb R^M: \theta_m\geq0, m=1,\ldots,M, \norm{\vtheta}_p\leq1\}$ and, by (B1), $\W\in\mathbb R^{eT}$, $e=e_1+\cdots+e_M$.
    Note that \eqref{eq:prob1} can be written unconstrained as
    \begin{equation}\label{eq:prob2}
      \inf_{\W,\thinspace\vtheta} ~ f(\W,\vtheta)  \thinspace, \quad \text{\rm where} ~~ f(\W,\vtheta) ~:=~ f_0(\W,\vtheta) + f_1(\W) + f_2(\vtheta)\thinspace,
  \end{equation}
  by putting 
    $$ f_0(\W,\vtheta) ~:=~ \frac{1}{2}\sum_{m=1}^M\frac{\Vert W_m\Vert^2_{Q_m}}{\theta_m} \thinspace+\thinspace I_{\{\vtheta\succ\zero\}}(\vtheta)     $$
    as well as
    \begin{equation}\label{def:f12}
       f_1(\W)~:=~C\sum_{i=1}^nl\bigg(y_i\sum_{m=1}^M\big\langle\w_{m\tau(i)},\varphi_m(x_i)\big\rangle\bigg)\thinspace, \quad 
         f_2(\vtheta)~:=~I_{\{\norm{\vtheta}_p\leq1\}}(\vtheta)\thinspace,
    \end{equation}
    where $I$ is the indicator function, $I_S(s)=0$ if $s\in S$ and $I_S(s)=\infty$ elsewise. 
    Note that we use the shorthand $\vtheta\succ\zero$ for $\theta_m>0$, $m=1,\ldots,M$.

  Assumption (B4) ensures that applying the block coordinate descent method to \eqref{eq:prob1} 
    and \eqref{eq:prob2} problems yields precisely the sequence of iterates.
    Thus, in order to prove the corollary, it suffices to validate that \eqref{eq:prob2} fulfills Assumptions (A1)--(A4) in Proposition~\ref{prop:tseng}.

  \smallskip
    \textsc{Validity of (A1)}~~
  Recall that Algorithm~\ref{alg:blueprint} is initialized with $\W^0=\zero$ and $\theta_m^0=\sqrt[p]{1/M}$, $m=1,\ldots,M$, so it holds
    \begin{equation}\label{eq:finite}
       f(\W^0,\vtheta^0) ~=~ \underbrace{f_0(\W^0,\vtheta^0)}_{=0} ~+ \underbrace{f_1(\W^0)}_{=Cn\thinspace\loss(0)} 
         +~  \underbrace{f_2(\vtheta^0 )}_{=0} ~=~Cn\thinspace\loss(0)~<~\infty \thinspace, 
  \end{equation}
  hence $f\not\equiv\infty$, so $f$ is proper. Furthermore, $\dom(f_0)=\{(\W,\vtheta): \vtheta\succ\zero\}$ is convex, and $f_0$ is convex on 
    $\dom(f_0)$, so $f_0$ is a convex function. By (B2), the loss function $l$ is convex, so $f_1$ is a convex function. The domain 
    $\dom(f_2)=\{\vtheta: \norm{\vtheta}_p\leq1\}$ is convex, and $f_2\equiv0$ on its domain, so $f_2$ is a convex function.
    Thus the sum $f=f_0+f_1+f_2$ is a convex function, which shows (A1).

  \smallskip
    \textsc{Validity of (A2)}~~
    Let $(\widetilde{\W},\widetilde{\vtheta})\in\mathcal A^0:=\{(\W,\vtheta):f(\W,\vtheta)\leq f(\W^0,\vtheta^0)\}$. 
    We have $f_0,f_1,f_2\geq0$, so, for all $m=1,\ldots,M$,
    \begin{align}\label{eq:conv_aux}
     \begin{split}
       \frac{\Vert\widetilde{W}_m\Vert_{Q_m}}{2\theta_m} &~\leq~ f_0(\widetilde{\W},\widetilde{\vtheta}) ~\leq~ f_0(\widetilde{\W},\widetilde{\vtheta}) + \underbrace{f_1(\widetilde{\W})}_{\geq0} 
        + \underbrace{f_2(\widetilde{\vtheta})}_{\geq0} \thinspace\stackrel{\rm by\thinspace\eqref{eq:prob2}}{=}\thinspace f(\widetilde{\W},\widetilde{\vtheta}) 
            \\
       & ~\leq~ f(\W^0,\vtheta^0)  \thinspace\stackrel{\rm by\thinspace\eqref{eq:finite}}{\leq}\thinspace Cn\thinspace\loss(0)\thinspace,
        \end{split}
  \end{align}
    which implies $\Vert\widetilde{W}_m\Vert^2_{Q_m}\leq 2\theta_mCn\thinspace\loss(0)$. 
    Similar, because $f_0\geq0$, we have
    $f_2(\widetilde{\W},\widetilde{\vtheta})\leq Cn\thinspace\loss(0)<\infty$, which, by \eqref{def:f12}, implies 
    $\Vert\widetilde{\vtheta}\Vert_p\leq1$ and thus $\widetilde{\theta}_m\leq1$, $m=1,\ldots,M$. Hence, by \eqref{eq:conv_aux}, 
    $\Vert \widetilde{W}_m\Vert^2_{Q_m}\leq\thinspace 2Cn\thinspace\loss(0)$, $m=1,\ldots,M$.
    Because $Q_1,\ldots,Q_M$ are positive definite, $\nu:=\min_{m=1,\ldots,M}\tr(Q_m)>0$.
    Thus, for any $m=1,\ldots,M$, 
    \begin{eqnarray*}
       \Vert\widetilde{W}_m\Vert^2 &=& \tr(\widetilde{W}_m^*\widetilde{W}_m) 
                ~ =~ \tr(\widetilde{W}_m^*\widetilde{W}_m)\tr(Q_m)/\tr(Q_m) ~\leq~ \tr(\widetilde{W}_m^*W_mQ_m)/\tr(Q_m) \\
          &\leq& \nu^{-1}\thinspace\tr(\widetilde{W}_m^*\widetilde{W}_mQ_m) ~=~ \nu^{-1}\thinspace\tr(\widetilde{W}_mQ_m\widetilde{W}_m^*) 
                  ~=~ \nu^{-1}\thinspace\Vert \widetilde{W}_m\Vert_{Q_m}^2 ~\leq~ 2\nu^{-1} Cn\thinspace\loss(0)  \thinspace. 
  \end{eqnarray*}
    Thus 
    $$ \Vert(\widetilde{\W},\widetilde{\vtheta})\Vert^2 ~=~ \Vert\widetilde{\W}\Vert^2 + \Vert\widetilde{\vtheta}\Vert^2 ~\leq~ 2\nu^{-1} CMn\thinspace\loss(0) + M ~<~ \infty\thinspace. $$
    Thus $\sup_{(\widetilde{W},\widetilde{\vtheta})\in\mathcal A^0}\Vert(\widetilde{W},\widetilde{\vtheta})\Vert<\infty$, which shows
    that $\mathcal A^0$ is bounded. Furthermore, $\mathcal A^0\subset\dom(f)=\dom(f_0)\cap\dom(f_1)\cap\dom(f_2)$ and $f_0,f_1,f_2$ are continuous on their respective domains.
    Thus $f$ is continuous on $\dom(f)$ and thus also on its subset $\mathcal A^0$. It holds $\mathcal A^0 = f^{-1}\big(]-\infty,f(\W^0,\vtheta^0)]\big)$, i.e., $\mathcal A^0$ is the
    preimage of closed set under a continuous function; thus $\mathcal A^0$ is closed.  Any closed and bounded subset of $\mathbb R^d$ is compact. 
    Thus $\mathcal A^0$ is compact, which was to show.
    
    \smallskip
    \textsc{Validity of (A3) and (A4)}~~
    Clearly, $\dom(f_0)=\{(\W,\vtheta): \vtheta\succ\zero\}$ is open and $f_0$ is continuously differentiable on $\dom(f_0)$. Thus it is G\^ateaux differentiable on $\dom(f_0)$.
    Finally, assumption (A4) is trivially fulfilled as Algorithm~\ref{alg:blueprint} employs a simple alternating rule for traversing the blocks of coordinates.
    
    \smallskip
    In summary, Proposition~\ref{prop:tseng} can thus be applied to Problem \ref{prob:primal}, which yields the claim of the corollary.
\end{proof}

\begin{remark}
  In this paper, we experiment on finite-dimensional string kernels,
  so Assumption (B1) is naturally fulfilled. Note that, more
  generally, $\phi(x_i)\in\R^{d_m}$ for all $i=1,\ldots,n$,
  $m=1,\ldots,M$, can be enforced also for infinite-dimensional
  kernels, as, for any finite sample $x_1,\ldots,x_n$, there exists a
  $n$-dimensional feature representation of the sample that can be
  explicitly computed in terms of the empirical kernel map
  \citep{SchMikBurKniMueRaeSmo99}.
\end{remark}

%
%


\section{Applications}\label{experiments}

We demonstrate the performance of different facets of our framework
with several experiments ranging from well-controlled toy data to a
large scale experiment on a highly relevant genomes data set, where we
combine data from a diverse set of organisms using multitask
learning. We start with a review of our prior experimental work based
on algorithms that are closely related to the ones described in this
work.

\subsection{Previous work}

The theoretical framework presented in this paper is a generalization
of the methods successfully used in our previous work.  Special cases
of the above framework were investigated in the context of genomic
signal prediction~\citep{NIPS2008_0510,widmer2010leveraging},
sequence segmentation with structured output learning~\citep{Goe11},
computational
immunology~\citep{widmer2010novel,Widmer2010,Toussaint2010} and
problems from biological
imaging~\citep{Lou2012b,widmer2014graph,xinghua14}.
Further, we have investigated an efficient algorithm to solve special
cases of our method on a large number of machine learning data sets
in~\citet{Widmer2012}.
We have previously summarized some of our earlier work
in~\citep{widmer2012multitask,Widmer2013ai,festschrift}.
%
%
\begin{figure}[h!]
  \centering
  \includegraphics[width=0.70\textwidth]{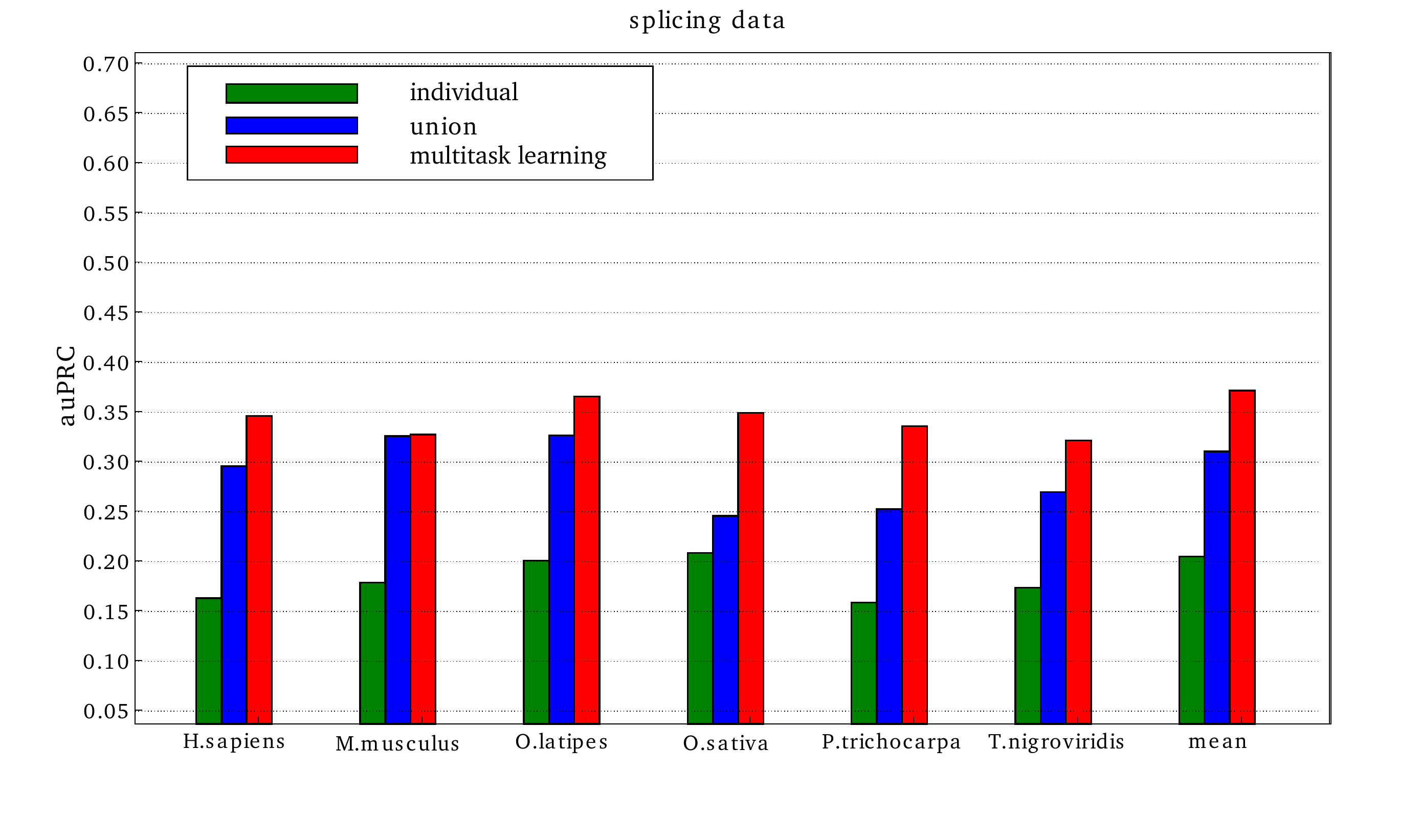}\vspace*{-2ex}
  \caption{Results from multitask learning on several organisms. Shown
    is a subset of the results reported in
    \citet{widmer2010leveraging}, where we combined splice site data
    from $15$ organisms. We compared a multitask learning approach to
    baseline methods \emph{individual} (each task is learned
    independently) and \emph{union} (all data is simply pooled). As
    for multitask learning, we used only a single, fix similarity
    measure, which we inferred from the evolutionary history of the
    organisms at hand. These and other results in
    \citet{NIPS2008_0510,widmer2010leveraging,Goe11,widmer2010novel,Widmer2010,Toussaint2010,Lou2012b,widmer2014graph,xinghua14}
    illustrate the power multitask learning in related tasks in
    computational biology.
   \label{fig:old_results}}
\end{figure}
%
%
An example of earlier results from~\citet{widmer2010leveraging} is
given in Figure~\ref{fig:old_results}. It illustrates an application
of the MTL algorithm to a case where we have multiple datasets
associated with 15 organisms. Their evolutionary relationship is
assumed to be known and is used for informing task relatedness in the
algorithm that is described in Section~\ref{graphreg} and
\citet{widmer2010leveraging}. This experiment exemplifies the
successful application of MTL to applications in computational biology
for the joint-analysis of multiple related problems.

In the two experiments that will be described in the sequel, we will
go beyond our previous work by investigating our framework in its full
generality.

\subsection{Experiments on Biologically Motivated Controlled Data}

In this section, we evaluate Hierarchical MT-MKL as described in
Section~\ref{hmtmkl} on an artificial data set motivated by biological
evolution.
At the core of this example is the binary classification of examples
generated from two $100$-dimensional isotropic Gaussian distributions
with a standard deviation of $\sigma=20$.
The difference of the mean vectors $\mu_{pos}$ and $\mu_{neg}$ is
captured by a difference vector $\mu_d$. We set $\mu_{pos} = 0.5
\mu_d$ and $\mu_{neg} = -0.5 \mu_d$.
To turn this into a MTL setting, we start with a single $\mu_d = (1,
\dots, 1)^T$ and apply mutations to it.
These mutations correspond to flipping the sign of $m$ dimensions in
$\mu_d$, where $m=5$.
Inspired by biological evolution, mutations are then applied in a
hierarchical fashion according to a binary tree of depth $4$
(corresponding to $2^4 = 32$ leaves).
Starting at the root node, we apply subsequent mutations to the
$\mu_d$ at the inner nodes of the hierarchy and work down the tree
until each leaf carries its own $\mu_d$.  We sample $10$ training
points and $1,000$ test points for each class and for each of the $32$
tasks.  The similarity between the $\mu_d$ at the leaves is computed
by taking the dot product between all pairs and is shown in
Figure~\ref{tsm}.  Clearly, this information is valuable when deciding
which tasks (corresponding to leaves in this context) should be
coupled and will be referred to as the \emph{true} task similarity
matrix in the following.
We use Hierarchical MT-MKL as described in Section~\ref{hmtmkl} by
creating adjacency matrices for each inner node and subsequently
learning a weighting using MT-MKL.

%
We compare MT-MKL with $p={1,2,3}$ to the following baseline methods:
\emph{Union} that combines data from all tasks into a single group,
\emph{Individual} that treats each task separately and \emph{Vanilla
  MTL} that uses MTL with the same weight for all matrices.  We report
the mean (averaged over tasks) ROC curve for each of the above methods
in Figure~\ref{auc}.

\begin{figure}[!ht]
  \centering
  \subfigure[True Task similarity matrix (see main text)]{
    \label{tsm}
    \includegraphics[width=0.47\textwidth]{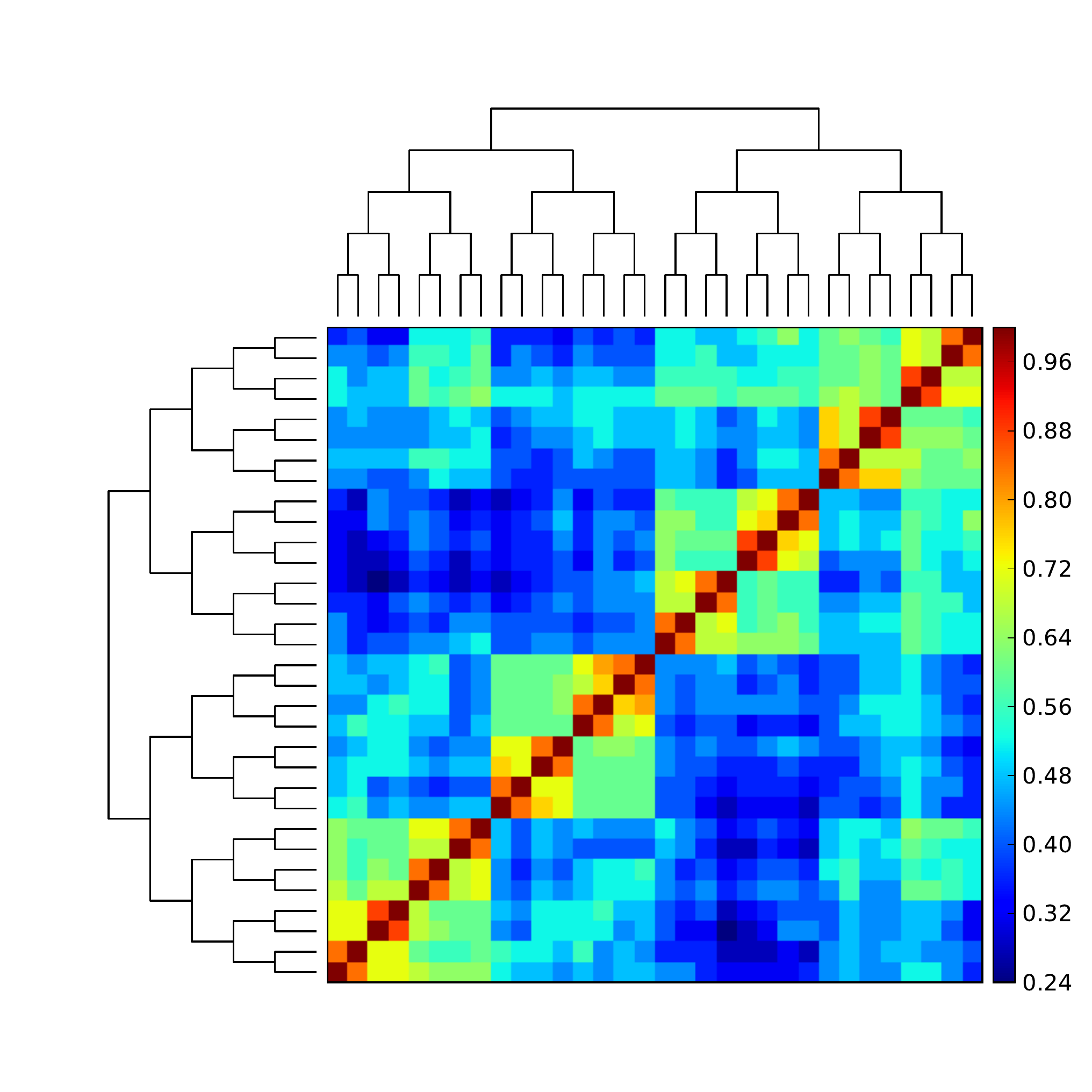}
  }
  \subfigure[Performance of Hierarchical MT-MKL \emph{vs.}\ Baseline methods]{
    \label{auc}
    \includegraphics[width=0.47\textwidth]{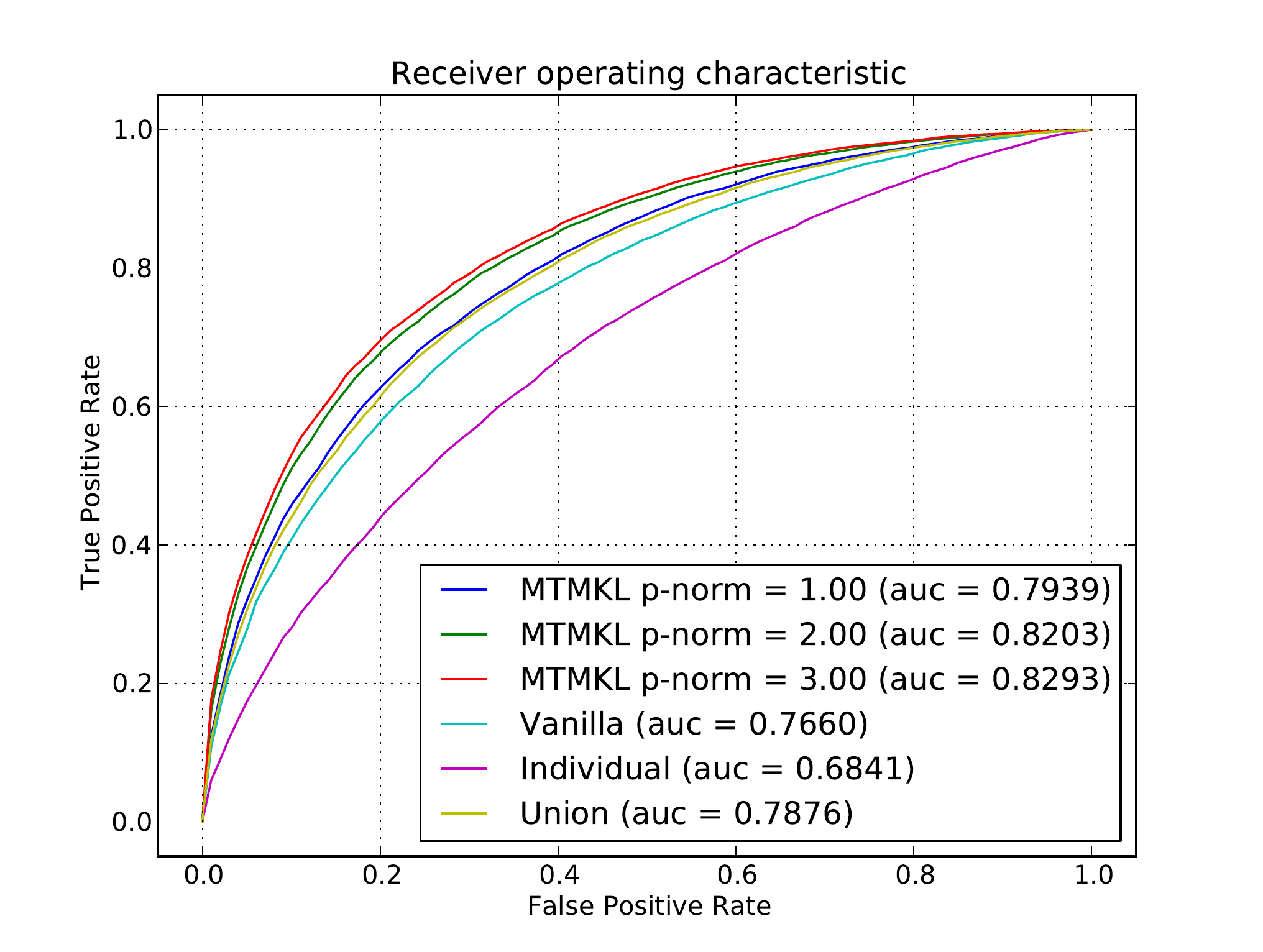}
  }
%
  \caption{Illustration of Hierarchical MT-MKL on an artificial
    dataset: In \ref{tsm}, we show the similarity matrix between all
    $32$ tasks as generated by a biologically inspired scheme, where
    generating parameters are \emph{mutated} according to a given tree
    structure (see main text for details). Comparison of MT-MKL to
    baselines \emph{Vanilla MTL}, \emph{Union}, \emph{Individual} is
    shown in \ref{auc}, where ROC curves are averaged over the $32$
    tasks for each method. MT-MKL with $p=2$ and $p=3$ perform best
    for this task.
   \label{fig:results}}
\end{figure}
From Figure~\ref{auc} we observe that the baseline \emph{Individual}
performs worst by a large margin, suggesting that combining
information from several tasks is clearly beneficial for this data
set.
Next, we observe that a simple way of combining tasks
(i.e., \emph{Union}) already considerably improves
performance. Furthermore, we observe that learning weights of
hierarchically inferred task grouping in fact improves performance
compared to \emph{Vanilla} for non-sparse MT-MKL (i.e., $p=2,3$).
Of all methods, non-sparse MT-MKL is most accurate for all recall values. 

\subsection{Genomic Signal - Transcription Start Site (TSS) Prediction}

In this experiment, we consider an application from genome sequence
analysis.  The goal is to accurately identify the genomic signal
called transcription start site (TSS) based on the surrounding genomic
sequence. TSS is the genomic location where transcription, the process
whereby the RNA copies are made from regions of the genome, is
initiated at the genome sequences.
We have obtained genomic data from ENSEMBL~\citep{Hubbard2002ensembl},
a community resource that brings together genomic sequences and their
annotations.  From this, we compiled a data set for nine organisms
(\textit{E.\ caballus}, \textit{C.\ briggsae}, \textit{M.\ musculus},
\textit{C.\ elegans}, \textit{D.\ rerio}, \textit{D.\ simulans},
\textit{V.\ vinifera}, \textit{A.\ thaliana}, and
\textit{H.\ sapiens}), where we took annotated instances of
transcription starts as positive examples and sequences around
randomly selected positions in the genome as background.  We use our
framework to jointly learn models for different organisms, treating
different organisms as different tasks.

\paragraph*{Task similarity}

To generate an initial task similarity matrix, we extracted the
phylogenetic similarity between different organisms based on their
genomic sequences.
%
%
In particular, we computed the Hamming distance between well-conserved
16S ribosomal RNA regions (i.e., stretches of genomic sequence with
low degree of change during evolution) between different classes of
organisms~\citep{Isenbarger2008rRNA}.
Subsequently, we either used this similarity directly in our multitask
learning algorithms (MTL) or attempted to refine it further using
MT-MKL.
%
%
%
To create a set of task similarities to be weighted by MT-MKL, we
applied exponential transformations to the base task similarity at
different length-scales ($\sigma=\{0.1, 7.55, 15.0\}$; see
  Section~\ref{smooth_MTMKL}).
%
%
%

\paragraph*{Experimental Setup and Results}

We have collected 4,000 TSS signal sequences for each organism, which
includes 1,000 positive and 3,000 negative label sequences for
training and testing. Both ends of the TSS signal label sequence
consist of $1,200$ flanking nucleotides.  On this data set, we
evaluated the two baseline methods, MTL and MT-MKL.
In the used evaluation scheme, we split the data in training set,
validation set and testing set for each organism. We use ten splits.
The best regularization constant is selected on the validation split
for each organism. 
In Figure~\ref{fig:tssresults} we report the average area under the
ROC curve (AUC) over the ten test sets, for each of which the best
regularization parameter was chosen on a separate evaluation set.

\begin{figure}[h]
\centering 
    \includegraphics[width=0.77\textwidth]{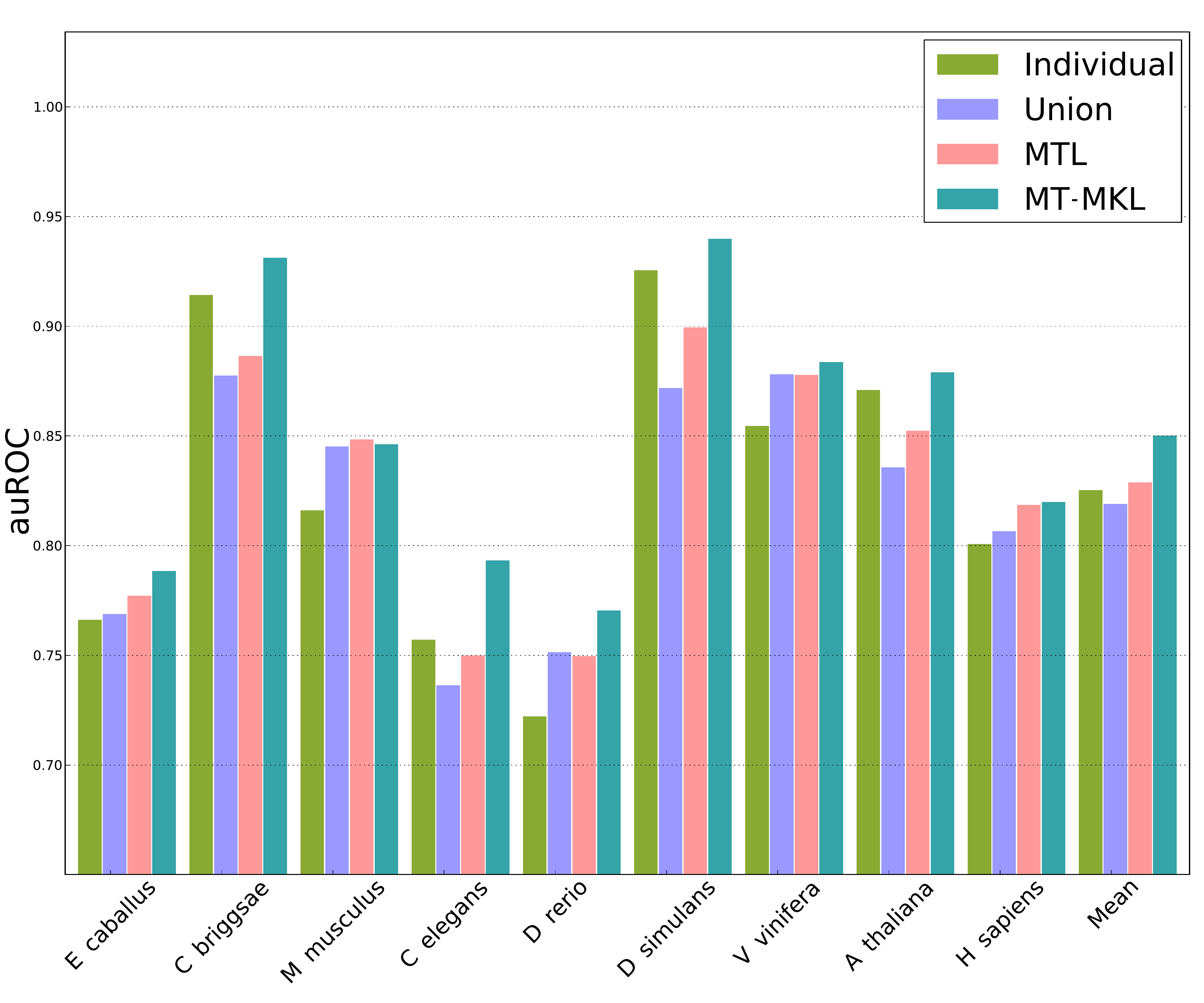}
  \caption{Average AUC achieved by the proposed MT-MKL as well as the
    baseline methods, on the gene-start dataset (TSS). MT-MKL improves
    the mean accuracy considerably. In addition, the accuracy of
    MT-MKL is best in eight out of the nine
    organisms. \label{fig:tssresults}}
\end{figure}

From Figure~\ref{fig:tssresults}, we observe that four out of nine
organisms the single-task SVM (individual) outperforms the SVM that is
trained on training instances from all organisms pooled (union).  From
which we conclude that the learning tasks are substantially
dissimilar. On the other hand, we observe that for some organisms
(\textit{M.\ musculus}, \textit{D.\ rerio}, \textit{V.\ vinifera}, and
\textit{H.\ sapiens}), there is an improvement by union over
individual, which indicates that these tasks are more similar than the
remaining tasks.  This is an indicator that MTL may be beneficial for
this data. See also discussion in \citet{festschrift}. Indeed, MTL
improves (at least marginally) over \emph{Union} and \emph{Individual}
in seven and five out of nine organisms, respectively.  But it is
surpassed by \emph{Individual} for three organisms
(\textit{A.\ thaliana}, \textit{C.\ briggsae}, \texttt{C.\ elegans},
and \textit{D.\ simulans}). While the overall performance of of MTL is
slightly better than \emph{Union} and \emph{Individual}, the
differences are minor which we attribute a possibly suboptimally
chosen task similarity matrix. (In fact, practically speaking, we find
that selecting a good task similarity matrix is the most difficult
aspect of Multitask learning.)

%
%
The proposed MT-MKL on the other hand, improves over individual on
eight out of nine organisms (and is not much worse on the nineth
task).  It improves over MTL by close to $5\%$ AUC for some
organisms. On average, it performs about $2.5\%$ better than any other
considered algorithm.  MT-MKL achieves this by refining task
similarities and thus is able to improve classification performance.

In summary, we are able to demonstrate that multitask learning and
MT-MKL strategies are beneficial when combining information from
several organisms and we believe that this setting has potential for
tackling future prediction problems in computational biology, and
potentially also to other application domains of multitask and
multiple kernel learning such as computer
vision~\citep{Lou2012b,kloft2009feature,binder2012insights,widmer2014graph,xinghua14}
and computer
security~\citep{kloft2008automatic,kloft2012security,gornitz2013toward}.

\section{Conclusion}\label{conclusions}

We presented a general regularization-based framework for Multi-task
learning (MTL), in which the similarity between tasks can be learned
or refined using $\ell_p$-norm Multiple Kernel learning (MKL).
Based on this very general formulation (including a general loss
function), we derived the corresponding dual formulation using Fenchel
duality applied to Hermitian matrices.
We showed that numerous established MTL methods can be derived as
special cases from both, the primal and dual of our formulation.
Furthermore, we derived an efficient dual-coordinate descend
optimization strategy for the hinge-loss variant of our formulation
and provide convergence bounds for our algorithm. Combined with our
efficient integration into the SHOGUN toolbox using the COFFIN feature
hashing framework, the approach could be used to process a large
number of training points.  The solver can also be used to solve the
vanilla $\ell_p$-norm MKL problem in the primal very efficiently, and
potentially extended to more recent MKL approaches
\citep{cortes2013learning}.
Our solvers including all discussed special cases are made available
as open-source software as part of the SHOGUN machine learning
toolbox.

In the experimental part of this paper, we analyzed our algorithm in
terms of predictive performance and ability to reconstruct task
relationships on toy data, as well as on problems from computational
biology.  This includes a study at the intersection of multitask
learning and genomics, where we analyzed $9$ organisms jointly.  In
summary, we were able to demonstrate that the proposed learning
algorithm can outperform baseline methods by combining information
from several organisms.

In the future we would investigate the theoretical foundations of the
approach (a good starting point to this end is the work by
\cite{kloft2011local,kloft2012convergence}), extensions to structured
output prediction \citep{Goe11}, and to apply the method to further
problems from computational biology and the biomedical domain.  These
settings have great potential; for instance, a Bayesian adaption of
our approach was very recently shown to be the leading model in an
international comparison of $44$ drug prediction methods for breast
cancer \citep{costello2014community}.

\paragraph{Acknowledgements} We thank thank Bernhard Sch\"olkopf, Gabriele Schweikert,
  Alexander Zien and S\"oren Sonnenburg for early contributions and
  helpful discussions and Klaus-Robert M\"uller and Mehryar Mohri for
  helpful discussions.  This work was supported by the German Research
  Foundation (DFG) under MU 987/6-1 and RA 1894/1-1 as well as by the
  European Community's 7th Framework Programme under the PASCAL2
  Network of Excellence (ICT-216886). Marius Kloft acknowledges
  support by the German Research Foundation through the grants KL
  2698/1-1 and 2698/2-1. Gunnar R\"atsch acknowledges additional
  support from the Sloan Kettering Institute.

\appendix

\section{Fenchel Duality in Hilbert Spaces}\label{app:fenchel}

In this section, we review Fenchel duality theory for convex functions
over real Hilbert spaces.  The results presented in this appendix are
taken from Chapters 15 and 19 in \cite{BauCom11}. For complementary
reading, we refer to the excellent introduction of
\cite{BauLuc12}. Fenchel duality for machine learning has also been
discussed in \cite{Rifkin2007} assuming Euclidean spaces.  We start
the presentation with the definition of the convex conjugate function.

\begin{definition}[Convex conjugate]
  Let $\Hilb$ be a real Hilbert space and let
  $g:\Hilb\rightarrow\R\cup\{\infty\}$ be a convex function.  We
  assume in the whole section that $g$ is proper, that is,
  $\{\w\in\Hilb\thinspace|\thinspace g(\w)\in\R\}\neq\emptyset$.  Then
  the convex conjugate $g^*:\Hilb\rightarrow\R\cup\{\infty\}$ is
  defined by $g^*(\w)=\sup_{\v\in\Hilb} \langle\v,\w\rangle - g(\v)$.
\end{definition}

\noindent\smallskip
As the convex conjugate is a supremum over affine functions, it is convex and lower semi-continuous. We have the beautiful duality
$$ g=g^{**} \thinspace\Leftrightarrow\thinspace \begin{cases} g\text{
    is convex and} \\ \text{lower
    semi-continuous}\thinspace. \end{cases} $$ This indicates that the
``right domain'' to study conjugate functions is the set of convex,
lower semi-continuous, and proper (``ccp'') functions.  In order to
present the main result of this appendix, we need the following
standard result from operator theory.

\begin{proposition}[Definition and uniqueness of the adjoint map]
  Let $\Hilb$ be a real Hilbert space and let
  $A:\Hilb\rightarrow\widetilde{\Hilb}$ be a continuous linear map.
  Then there exists a unique continuous linear map
  $A^*:\widetilde{\Hilb}\rightarrow\Hilb$ with $\langle
  A(\w),\valpha\rangle=\langle\w,A^*\valpha\rangle$, which is called
  \text{\rm adjoint map} of $A$.
\end{proposition}

\smallskip\noindent For example, in the Euclidean case, we have
$\Hilb=\R^m$, $\widetilde{\Hilb}=\R^n$, and $A\in\mathbb\R^{m\times
  n}$ so that simply the transpose $A^*=A^\top\in\mathbb\R^{n\times
  m}$.  We now present the main result of this appendix, which is
known as \emph{Fenchel's duality theorem}:

\begin{theorem}[Fenchel's duality theorem]\label{thm:fenchel}
  Let $\Hilb,\widetilde{\Hilb}$ be real Hilbert spaces and let
  $g:\Hilb\rightarrow\R\cup\set{\infty}$ and
  $h:\widetilde{\Hilb}\rightarrow\R\cup\set{\infty}$ be ccp.  Let
  $A:\Hilb\rightarrow\widetilde{\Hilb}$ be a continuous linear map.
  Then the primal and dual problems, 
  \begin{align*}
    & p^* = \inf_{\w\in\Hilb} ~ g(\w) + h(A(\w)) \\
    & d^* = \thinspace\sup_{\valpha\in\widetilde{\Hilb}} \thinspace -g^*(A^*(\valpha)) - h^*(-\valpha) \thinspace,
  \end{align*}
  satisfy \text{ \rm weak duality} (i.e., $d^*\leq p^*$). Assume,
  furthermore, that
  $A(\textrm{dom}(g))\cap\textrm{cont}(h)\neq\emptyset$, where
  $\text{\rm dom}(f):=\{\w\in\Hilb:g(\w)<\infty\}$ and $\text{\rm
    cont(h)}:=\{\valpha\in\widetilde{\Hilb}:
  h~\textrm{continuous~in}~\valpha\}$.  Then we even have \text{\rm
    strong duality} (i.e., $d^*=p^*$) and any optimal solution
  $(\w^\star,\valpha^\star)$ satisfies
  $$ \w^\star=\nabla g^*(A^*(\valpha^\star)) \thinspace, $$
  if $g^*\circ A^*$ is (G\^ateaux) differentiable in $\valpha^\star$.
\end{theorem}

\smallskip
\noindent When applying Fenchel duality theory, we frequently need to
compute the convex conjugates of certain functions. To this end, the
following computation rules are helpful.

\begin{proposition}\label{eq:comp_rules}
  The following computation rules hold for the convex conjugate:
  \begin{enumerate}
  \item Let $g:\Hilb\rightarrow\R\cup\set{\infty}$ be a proper convex
    function on a real Hilbert space $\Hilb$.  Then, for any
    $\lambda>0$ and $\w\in\Hilb$, we have $(\lambda g)^*(\w)=\lambda
    h^*(\w/\lambda)$.
  \item Furthermore, assume that $\Hilb=\Hilb_1\bigoplus\Hilb_2$ and
    $g(\w)=g_1(\w_1)+g_2(\w_2)$, where
    $g_1:\Hilb_1\rightarrow\R\cup\set{\infty}$ and
    $g_2:\Hilb_2\rightarrow\R\cup\set{\infty}$, are proper convex
    functions on Hilbert spaces $\Hilb_1$ and $\Hilb_2$, respectively.
    Then, for any $\w=(\w_1,\w_2)\in\Hilb_1\bigoplus\Hilb_2$, we have
    $g^*(\w)=g_1^*(\w_1)+g^*_2(\w_2)$.
  \end{enumerate}
\end{proposition}

\section{Conjugate of the Logistic Loss}\label{app:loss}    

The following lemma gives the convex conjugate of the logistic loss.

\begin{lemma}[Conjugate of Logistic Loss]\label{lemma:logconj}
  The conjugate of the logistic loss, defined as
  $l(a)=\log(1+\exp(-a)$, is given by
    $$l^*(a)=-t\log(-a)+(1+a)\log(1+a) \thinspace. $$
\end{lemma}

\begin{proof}
  By definition of the conjugate,
  $$ l^*(a) ~=~ \sup_{b\in\Real} ~ \underbrace{ab \thinspace-\thinspace \log(1+\exp(-b))}_{=:\psi(b)} \thinspace. $$
Note that the problem is unbounded for $a<-1$ and $a>0$. For $a\in]-1,0[$, the supremum is attained when $\psi'(b)=0$, which translates into $b=-\log(-a/(1+a))$ and $1+\exp(-b)=1/(1+a)$. Thus
$$ l^*(a) \thinspace=\thinspace -a\log(-a/(1+a)) -\log(1/(1+a)) \thinspace=\thinspace -a\log(-a) + (1+a)\log(1+a) \thinspace, $$
which was to show
\end{proof}

\bibliography{multitask}

\end{document}